\documentclass[11pt]{article}
\usepackage{amsmath,amsfonts,amsthm,mathrsfs,url}
\usepackage{color}
\usepackage{fancyhdr}
\usepackage[top=2.5cm, bottom=2.5cm, left=3.1cm,
right=3.1cm]{geometry}
\usepackage{graphicx}
\newtheorem{example}{Example}
\newtheorem{theorem}{Theorem}
\newtheorem{assumption}{Assumption}
\newtheorem{lemma}{Lemma}
\newtheorem{proposition}{Proposition}
\newtheorem{remark}{Remark}
\newtheorem{corollary}{Corollary}
\newtheorem{definition}{Definition}

\allowdisplaybreaks[4]
\numberwithin{equation}{section}
\def\dsum{\displaystyle\sum }
\def\dsup{\displaystyle\sup }

\def\dint{\displaystyle\int }
\def\diint{\displaystyle\iint }
\def\begeqn{\begin{equation}}
\def\endeqn{\end{equation}}
\def\begth{\begin{theorem}}
\def\endth{\end{theorem}}
\def\begprop{\begin{proposition}}
\def\endprop{\end{proposition}}
\def\begcor{\begin{corollary}}
\def\endcor{\end{corollary}}
\def\begdef{\begin{definition}}
\def\enddef{\end{definition}}
\def\beglemm{\begin{lemma}}
\def\endlemm{\end{lemma}}
\def\begexm{\begin{example}}
\def\endexm{\end{example}}
\def\begrem{\begin{remark}}
\def\endrem{\end{remark}}
\def\begassum{\begin{assumption}}
\def\endassum{\end{assumption}}
\def\beg{\begin}
\def\ga{\alpha}

\def\gb{\beta}

\def\gga{{\gamma}}

\def\gd{\delta}
\def\gD{\Delta}
\def\gep{\varepsilon}

\def\gth{\theta}
\def\gTh{\Theta}

\def\gk{\kappa}
\def\gl{\lambda}

\def\gs{\sigma}

\def\bz{{\bf z}}

\def\O{\mathcal{O}}

\def\gTh{{\Theta}}

\def\N{\mathbb{N}}
\def\R{\mathbb{R}}

\def\X{\mathcal{X}}
\def\Y{\mathcal{Y}}
\def\Z{\mathcal{Z}}
\def\E{\mathcal{E}}

\def\H{\mathcal{H}}

\def\F{\mathcal{F}}
\def\EX{{\mathbb{E}}}

\def\beg{\begin}

\def\gl{\lambda}
\def\R{\mathbb{R}}

\def\wtE{\widetilde{\E}}

\def\wtd{\widetilde{d}}
\def\ts{\widetilde{s}}

\def\wtx{\widetilde{x}}
\def\wty{\widetilde{y}}
\def\wtz{\widetilde{z}}
\def\wtD{\widetilde{D}}

\def\gtk{\widetilde{\kappa}}

\begin{document}

\title{Unregularized Online Learning Algorithms with General Loss Functions\thanks{Corresponding author: Yiming Ying.   Email: yying@albany.edu}}
\author{Yiming Ying$^\dag$ and Ding-Xuan Zhou$^\ddag$\\
\\
$^\dag$Department of Mathematics and Statistics\\
State University of New York at Albany, Albany, NY, 12222, USA\\
$^\ddag$Department of Mathematics, City University of Hong Kong\\
Kowloon, Hong Kong, China}

\date{}

\maketitle

\begin{abstract} In this paper, we consider unregularized online learning algorithms in a Reproducing Kernel Hilbert Spaces (RKHS). Firstly, we derive explicit convergence rates of the unregularized online learning algorithms for classification associated with a general $\ga$-activating loss (see Definition 1 below). Our results extend and refine the results in \cite{YP} for the least-square loss and the recent result \cite{Bach} for the loss function with a Lipschitz-continuous gradient. Moreover, we establish a very general condition on the step sizes which guarantees the convergence of the last iterate of such algorithms. Secondly, we establish, for the first time, the convergence of the unregularized pairwise learning algorithm with a general loss function and derive explicit rates under the assumption of polynomially decaying step sizes. Concrete examples are used to illustrate our main results. The main techniques are tools from convex analysis, refined inequalities of Gaussian averages \cite{BM}, and an induction approach.

\noindent{\bf Keywords:} Learning theory, Online learning, Reproducing kernel Hilbert
space, Pairwise learning, Bipartite ranking
\end{abstract}

%%%%%%%%%%%%%%%%%%%%%%%%%%%%%%%%%%%%%%%%%%%%%%%%%%%%%%%%%%%%%%%%%%%%%%
%%%%%%%%%%%%%%%%%%%%%%%%%%%%%%%%%%%%%%%%%%%%%%%%%%%%%%%%%%%%%%%%%%%%%
\parindent=0cm
%%%%%%%%%%%%%%%%%%%%%%%%%%%%%%%%%%%%%

\section{Introduction}

Let the input space $\X$ be a complete metric space and the output space $\Y=\{ \pm 1\}.$
In the standard framework of learning theory \cite{CZ,Steinwart}, one considers the problem of learning from a set
of examples $\bz = \{z_i=(x_i,y_i) \in \X \times \Y: i=1,2,\ldots, T\}$  which are independently and
identically distributed (i.i.d.) according to an unknown distribution $\rho$ on $\Z = \X\times \Y.$

In the task of classification, a univariate loss function $\phi(yf(x))$ measures the error when $f(x)$ is used to predict the true label $y. $ In this case, one aims to find a predictor in a hypothesis space  to minimize the following true (generalization) error which is defined, for a  function $g:\X \to \R$, by $$\E(g) = \diint_\Z   \phi(yg(x)) d\rho(x,y).$$
In contrast to the task of classification, pairwise learning problems involve a  pairwise loss function
$\phi((y-y')f(x,x'))$ for a hypothesis function $f: \X\times \X \to \R.$ Notable examples of pairwise learning tasks include bipartite ranking
\cite{Agarwal,Clem,Rejchel}, similarity and metric learning \cite{CGY,Weinberger1}, AUC maximization \cite{Zhao} and gradient learning \cite{MW,MZ,YWC}.  The aim of pairwise learning is to minimize the  true error which is defined, for a pairwise function $f: \X \times \X \to \R$, by
$$\wtE(f) = \diint_{\Z\times \Z} \phi((y-y')f(x,x'))d\rho(x,y)d\rho(x',y').$$

In this paper, we consider online learning algorithms for both classification and pairwise learning tasks in a Reproducing Kernel Hilbert Space (RKHS). Specifically, let $G: \X \times \X \to \R$ be a {\em Mercer kernel}, i.e. a continuous, symmetric and positive semi-definite kernel, see e.g. \cite{CZ,Steinwart}. According to \cite{Aron}, the RKHS ${\mathcal H}_{G}$ associated with kernel $G$ is defined to be
the completion of the linear span of the set of functions $\{G_x(\cdot):= G(x,\cdot): x\in \X\}$
with an inner product satisfying the reproducing property, i.e., for any $x',x\in \X$,  $\langle G_x, G_{x'} \rangle_G =  G(x,x').$ Similarly, for pairwise learning, we assume that the pairwise function $f:\X \times \X \to \R$ is from an RKHS defined on the domain $\X^2:=\X \times \X$ with a (pairwise) kernel $K:  \X^2 \times \X^2 \to \R.$ Throughout this paper, we consider a specific family of loss functions called $\ga$-activating loss defined as follows.

\begin{definition}A function  $\phi: \R \to \R^+$ is called an $\ga$-activating loss  with some $\ga\in (0,1]$ if
it is convex and differentiable, $\phi'(0) < 0$, and $L:=\sup_{\ts,s\in\R}  {|\phi'(\ts)-\phi'(s)| / |\ts-s|^\ga}<\infty.$
\end{definition}

Our definition of $\ga$-activating loss follows \cite{WYZ} where the concept of the activating loss was first introduced. One can find in-depth discussions in \cite{Bartlett,Zhang} on loss functions for classification. Typical examples of $\ga$-activating losses includes $q$-norm loss \cite{Chen,Zhang} $\phi(s)=(1-s)^q_+=\max\{1-s,0\}^q$ for the support vector machine (SVM) classification with $1<q\le 2$, the least square loss $\phi(s)= (1-s)^2$ and the logistic regression loss $\phi(s) = \log(1+e^{-s}).$

The first purpose of this paper is to study the unregularized online learning algorithm for classification associated with a general  $\ga$-activating loss defined as follows.

\noindent{\bf Algorithm 1.} {\em Given the i.i.d. generated training data $\bz = \{z_i=(x_i,y_i): i=1,2,\ldots,T\}$, the unregularized online learning algorithm   is given by $g_1=0$ and, for any $1\le t\le T$,  \begin{equation}\label{eq:algorithm-1} g_{t+1} = g_t -\gga_t \phi'(y_t g_t(x_t)) y_t  G_{x_t}.\end{equation}
where $\{\gga_t>0: t\in \N\}$ is usually referred to as the step
size.}

Online learning algorithms for classification or regression have drawn much attentions \cite{Bach,SY,TY,YeZ,YP,YZ}. Most of them focused on regularized online learning algorithms, i.e.  $g_{t+1} = g_t - \gga_t (\phi'(y_t g_t(x_t)) y_t  G_{x_t} +\gl g_t).$ In particular, regularized online learning with  a fixed $\gl>0$ was studied in \cite{SY} for the least-square loss and in \cite{YZ} for the general loss function, and in \cite{TY,YeZ} for a time-varying regularization, i.e. $\gl = \gl(t)>0.$

Instead, we focus on deriving explicit convergence rates of the unregularized online learning algorithms (i.e. $\gl=0$) with a general $\ga$-activating loss.  Our results extend and refine those in \cite{YP} for the least-square loss and the recent result \cite[Theorem 4]{Bach} for the loss function with a Lipschitz-continuous gradient.  In contrast to the results \cite{YP,Bach} derived with the step sizes being chosen in the special form of $\O(t^{-\gth}),$ we will establish a very general condition on the step sizes which guarantees the convergence of the last iterate $g_{T+1}$ of Algorithm 1. Moreover, in the contrast to the proof in \cite{Bach}, we will soon see that  our new proof here is much simpler and more powerful to handle general loss functions.

The second purpose of this paper is to study the convergence of the last iterate of the following online pairwise learning algorithm, which is associated with an $\ga$-activating loss function and the RKHS $\H_K.$

\noindent{\bf Algorithm 2.} {\em Given the i.i.d. generated training data $\bz = \{z_i=(x_i,y_i): i=1,2,\ldots,T\}$, the unregularized online pairwise learning algorithm   is given by $f_1=f_2=0$ and, for any $2\le t\le T$,  \begin{equation}\label{eq:algorithm-2} f_{t+1} = f_t - {\gga_t\over
t-1} \sum_{j=1}^{t-1}\phi'((y_t-y_j)f_t(x_t,x_j)) (y_t-y_j)K_{(x_t,x_j)}.
\end{equation}}

Online pairwise learning involves non-i.d.d. pairs of examples, which introduces more difficulty than the analysis in the univariate case.  The research in this direction was recently conducted in \cite{Kar,Wang,YZ2}.  In particular,  in \cite{Kar,Wang} the convergence of the average of the iterates (i.e. ${1\over T}{\sum_{t=2}^{T+1} f_t }$) was established in the linear case by following online-to-batch conversion approach similar to those in the univariate case \cite{Cesa}. Recent work \cite{YZ2} focuses on Algorithm 2 with the least-square loss. However, the analysis techniques there
heavily depend on the nature of the least-square loss (e.g.  its derivative is a linear function) and do not apply to the general loss function.

In this paper, we establish, for the first time, the convergence of the last iterate of the unregularized pairwise learning algorithm  (Algorithm 2) with a general loss function and derive explicit rates under the assumption of polynomially decaying step sizes. Concrete examples are used to illustrate our main results. The main techniques are tools from convex analysis and refined inequalities related to the Gaussian averages \cite{BM}.

\section{Main Results}\label{sec:main-result}

In this section, we present our main results related to Algorithms 1 and 2.  The following theorem states a general convergence result for Algorithm 1.

\begin{theorem}\label{thm:conv-1} Assume that $\phi$ is $\ga$-activating with some $0<\ga\le 1 $ and let $\{g_t: t=1,\ldots, T+1\}$ be given by Algorithm 1. If the step sizes satisfy that $  \sum_{t=1}^\infty \gga_t^{1+\ga} <\infty,$
then   $\displaystyle\lim_{T\to \infty}   \EX\bigl[\E(g_{T+1})\bigr]$ exists. If, furthermore, $g_\H = \arg\inf_{g\in \H_G} \E(g)$ exits and  $\sum_{t=1}^\infty \gga_t =\infty$, then
$\displaystyle\lim_{T\to \infty}   \EX\bigl[\E(g_{T+1})\bigr] = \inf_{g\in \H_G} \E(g).$
\end{theorem}

By the above theorem,   the step sizes can be chosen in the form of $\gga_t = c \, t^{-\gth}$ with some $\gth\in ({1\over 1+\ga},1),$ and $c>0.$ Indeed, we can further derive the explicit convergence rate for the last iterate of Algorithm 1.

\begin{theorem}\label{thm:rate-1}Assume that $\phi$ is $\ga$-activating with some $0<\ga\le 1 $ and  $g_\H = \arg\inf_{g\in \H_G} \E(g)$ exits. Choose step sizes $\gga_t = {c\, t^{-\gth}}$ with some $\gth\in ({1\over 1+\ga}, 1)$ and $c>0.$  Then,
$$ \EX\bigl[\E(g_{T+1}) - \E(g_\H)\bigr]   =  C_{\gth,\ga,\H} {T^{-{\min({\ga \gth\over 2}, 1-\gth)}}},$$
where the constant $C_{\gth,\ga,\H}$ depends on $\gth, \ga, c$ and $\|g_\H\|_G$ (see its explicit form in the proof).
\end{theorem}
From the above theorem, the maximal rate for $\ga$-activating losses is of the form   $\O(T^{-{\ga\over \ga+2}})$ which is achieved by choosing $\gga_t = {c\, t^{-{2\over \ga+2}}}.$   When $\ga=1$, the rate is of $\O(T^{-{1\over 3}})$ which is consistent with that in \cite{Bach}. We can directly get the following examples from the above theorems, since  $\phi(t)=  (1-t)_+^q$ with $q\in (1,2]$  is a $(q-1)$-activating loss and  $\phi(t)=  \log(1+e^{-t})$  is a $1$-activating loss.

\begin{example} Let $\phi(t)=  (1-t)_+^q$ with $1<q\le 2$ and assume that $g_\H = \arg\inf_{g\in \H_G} \E(g)$ exits. Let $\{g_t: t=1,\ldots, T+1\}$ be given by Algorithm 1 with step sizes $\gga_t = {c\, t^{-\gth}}$ with some $\gth\in ({1\over q}, 1),$ and $c>0.$  Then,
$$ \EX\bigl[\E(g_{T+1}) - \E(g_\H)\bigr]   = \O\bigl({T^{-{\min({(q-1) \gth\over 2}, 1-\gth)}}}\bigr).$$
\end{example}

\begin{example} Let  $\phi(t)=  \log(1+e^{-t})$ and assume that $g_\H = \arg\inf_{g\in \H_G} \E(g)$ exits. Let $\{g_t: t=1,\ldots, T+1\}$ be given by Algorithm 1 with step sizes $\gga_t = {c\, t^{-\gth}}$ with some $\gth\in ({1\over 2}, 1),$ and $c>0.$  Then,
$$ \EX\bigl[\E(g_{T+1}) - \E(g_\H)\bigr]   = \O\bigl({T^{-{\min({\gth\over 2}, 1-\gth)}}}\bigr).$$
\end{example}

Now we turn our attention to the convergence rates of Algorithm 2.

\begin{theorem}\label{thm:rate-2}Assume $\phi$ is $1$-activating, and $f_\H = \arg\inf_{f\in \H_K} \wtE(f)$ exits. Let $\{f_t: t=1,\ldots, T+1\}$ be given by Algorithm 2 with step sizes $\gga_t = {c\, t^{-\gth}}$ with some $\gth\in ({1\over 2}, 1)$ and $0<c\le {1\over {4\gtk^2 L}}.$  Then, for any $\gd \in (0, \min(\gth-{1\over 2}, 1-\gth)),$ there holds
$$ \EX\bigl[\wtE(f_{T+1}) - \wtE(f_\H)\bigr]   =  \tilde{C}_{\gth,\gd,\H} \,{T^{-{\min({\gth\over 2}-{1\over 4}-{\gd\over 2}, 1-\gth-\gd)}}},$$
where the constant $D_{\gth,\ga, \H}$ depends on $\gth, \gd$ and $\|f_\H\|_G$ (see its explicit form in the proof).
\end{theorem}

If, moreover, the gradient of $\phi$ is uniformly bounded then the rate in the above theorem can further be improved.

\begin{theorem}\label{thm:rate-3}Under the same assumptions of Theorem \ref{thm:rate-2} and further assuming $|\phi'(s)|\le B < \infty$ for any $s\in \R,$ then, for any $\gd \in (0, \min({\gth\over 4}, 1-\gth))$, we have
$$ \EX\bigl[\wtE(f_{T+1}) - \wtE(f_\H)\bigr]   =  \bar{C}_{\gth,\gd, \H} {T^{-\min({\gth\over 4}-{\gd\over 2 }, 1-\gth-\gd)}},$$
where the constant $\bar{C}_{\gth,\H}$ depends on $\gth, \gd$ and $\|f_\H\|_G$ (see its explicit form in the proof).
\end{theorem}
From the above theorem, we see that the maximal rate for Algorithm 2 associated with an $\ga$-activating loss is arbitrarily close to $\O(T^{-{1\over 6}}).$ If,  moreover, the gradient of the loss function $\phi$ is uniformly bounded then the maximal rate is improved to
 $\O(T^{-{1\over 5}}).$ In particular, from the above theorem, we can immediately get the following examples since $\phi(t)=  (1-t)_+^2$ and $\phi(t)=  \log(1+e^{-t})$ are both $1$-activating loss functions, and the gradient of $\phi(t)=  \log(1+e^{-t})$ is uniformly bounded by one.

\begin{example} Let $\phi(t)=  (1-t)_+^2$ with $1<q\le 2$ and assume that $f_\H = \arg\inf_{f\in \H_K} \wtE(f)$ exits. Let $\{g_t: t=1,\ldots, T+1\}$ be given by Algorithm 2 with step sizes $\gga_t = {c\, t^{-\gth}}$ with some $\gth\in ({1\over 2}, 1)$ and $c>0.$  Then, for any $\gd \in (0, \min(\gth-{1\over 2}, 1-\gth)),$ there holds
$$ \EX\bigl[\wtE(f_{T+1}) - \wtE(f_\H)\bigr]   =  \O\Bigl({T^{-{\min({\gth\over 2}-{1\over 4}-{\gd\over 2}, 1-\gth-\gd)}}}\Bigr). $$
\end{example}

\begin{example} Let  $\phi(t)=  \log(1+e^{-t})$ and assume that $f_\H = \arg\inf_{f\in \H_K} \wtE(f)$ exits. Let $\{g_t: t=1,\ldots, T+1\}$ be given by Algorithm 2 with step sizes $\gga_t = {c\, t^{-\gth}}$ with some $\gth\in ({1\over 2}, 1),$ and $c>0.$  Then, for any $\gd \in (0, \min({\gth\over 4}, 1-\gth))$,
$$ \EX\bigl[\wtE(f_{T+1}) - \wtE(f_\H)\bigr]   =  \O\Bigl({T^{-\min({\gth\over 4}-{\gd\over 2 }, 1-\gth-\gd)}}\Bigr).$$
\end{example}

\section{Proofs of Main Results}

We derive some useful properties of the $\ga$-activating loss function $\phi,$ which play critical roles in proving main theorems. Some of them may be  of interest in their own rights.

\begin{proposition}\label{prop:1} Assume that $\phi: \R \to \R$ is  convex and its gradient is $\ga$-H\"{o}lder continuous, i.e. $L:=\dsup_{	\ts,s\in\R}  {|\phi'(	\ts)-\phi'(s)| / |	\ts-s|^\ga}<\infty.$  Then, for any $s, \ts\in \R$, the following properties hold true.

\item (a) $\phi(s) - \phi(	\ts) -\phi'(\ts) (s-\ts) \le {L \over 1+\ga} |s-\ts|^{1+\ga}.$

\item (b)   $\phi(\ts) \ge \phi(s) + \phi'(s) (\ts-s) +  {\ga L^{-{1\over \ga}}\over 1+\ga} |\phi'(s)- \phi'(\ts)|^{1+\ga \over \ga}.$

\item (c)  $ (\phi'(s) - \phi'(\ts)) (s-\ts) \ge {2\ga L^{-{1\over \ga }}\over 1+\ga }  |\phi'(s)-\phi'(\ts)|^{1+\ga \over \ga}.$

\item (d) If, moreover, $\phi(s)\ge 0$ for any $s\in \R$, then $|\phi'(s)|^{1+\ga \over \ga}   \le  {(1+\ga)^{1+{1\over \ga}}\over \ga} L^{{1\over \ga}}\,  \phi(s).$
\end{proposition}

\begin{proof}  Part (a) directly follows from the fact that the assumption that $|\phi'(s) - \phi'(\ts)| \le L |s-\ts|^\ga$ and the fact $$\phi(s) - \phi(\ts) -\phi'(\ts) (s-\ts) = \dint^{1}_0 (\phi'(\gth s + (1-\gth)\ts) - \phi'(\ts))(s-\ts) d\gth.$$

For part (b), let $\psi_s(\ts)  = \phi(\ts) - \phi'(s) \ts.$  Notice that $\psi_s(\cdot)$ is convex,  differentiable and its gradient $\psi_s'(\ts) = \phi'(\ts) - \phi'(s) $ is $\ga$-H\"{o}lder continuous. In addition, $\psi_s(\cdot)$ achieves the minimum at $s$ since $\psi_s'(s) =0.$ Hence, for $\gd = L^{1\over \ga}$,
$$\begin{array}{ll}\psi_s(s)  & \le  \psi_s \bigl(\ts - {1\over \gd}   {(\phi'(\ts)- \phi'(s)) |\phi'(\ts)- \phi'(s) |^{1-\ga \over \ga}}\bigr) \\
& \le \psi_s(\ts) + \psi_s'(\ts)\bigl( -  {1\over \gd}   {(\phi'(\ts)- \phi'(s)) |\phi'(\ts)- \phi'(s) |^{1-\ga \over \ga}}\bigr)  \\
& +  {L\over 1+\ga}\bigl|{1\over \gd}   {(\phi'(\ts)- \phi'(s)) |\phi'(\ts)- \phi'(s) |^{1-\ga \over \ga}} \bigr|^{1+\ga}\\
&   = \psi_s(\ts)-  {\ga L^{-{1\over \ga}}\over 1+\ga} |\phi'(s)- \phi'(\ts)|^{1+\ga \over \ga},\end{array}$$
where the second to last inequality used the fact that $\psi_s(\cdot)$ satisfies part (a).  By the definition of $\psi_s(\cdot)$, re-arranging the terms in the above estimation yields the desired result of part (b).

For part (c), switching the roles of $\ts,s$ in part (b) yields that
$$\phi(s) \ge \phi(\ts) + \phi'(\ts) (s-\ts) +  {\ga L^{-{1\over \ga}}\over 1+\ga} |\phi'(s)- \phi'(\ts)|^{1+\ga \over \ga}.$$
Adding part (b) and the above inequality implies part (c).

For part (d),  the case for $\ga=1$ was proved in \cite{Srebro}. We generalize their proof to the general case $0<\ga\le 1.$  Indeed, we only need to prove the case $\phi'(s)\neq 0.$  For any $s\in \R$, let $ r = s- ((1+\ga)L)^{-{1\over \ga}}|\phi'(s)|^{1\over\ga}  {\phi'(s)\over |\phi'(s)|}.$ By the mean-value theorem, there exists  $\xi$ in the range $(s, r)$ (if $\phi'(s)<0$) or $(r,s)$ (if $\phi'(s)>0$) such that $\phi(r) = \phi(s) + \phi'(\xi)(r-s).$  Hence,
$$\beg{array}{ll} 0\le \phi(r) & = \phi(s) + \phi'(s)(r-s) + (\phi'(\xi)-\phi'(s))(r-s) \\
& \le  \phi(s) + \phi'(s)(r-s) +  L|r-s||\xi-s|^\ga\\
& \le  \phi(s) + \phi'(s)(r-s) +  L|r-s|^{1+\ga} = \phi(s) - {\ga \over (1+\ga)^{1+{1\over \ga}}} L^{-{1\over \ga}}|\phi'(s)|^{1+\ga \over \ga},
\end{array}
$$
which completes the proof of part (d).
\end{proof}
We end this section with a comment on deriving the

\subsection{Proofs for the Convergence of Algorithm 1}
The main idea for proving the convergence of Algorithm 1 is to   derive a recursive inequality for the sequence $\{R_t := \EX[\wtE(g_t) - \wtE(g_\H)]:  1\le t\le T+1\}$ (i.e. the relationship between $R_{t+1}$ and $R_t$), and then apply induction on this inequality. To this end, we need to establish the boundedness of the learning sequence $\{g_t: t=1,2,\ldots, T+1\}$ generated by Algorithm 1. Throughout the paper, we use the conventional notion that  $\sum_{j=k}^t \gga_j^{1+\ga}=0$ whenever $t<k.$ Denote $\gk = \sup_{x\in \X} \sqrt{G(x,x)}.$

\beglemm\label{lemm:bound-1}Let  $\{g_t:  t=1,\ldots, T+1\}$ be generated by Algorithm 1. Then,
$$\EX\bigl[\E(g_{t+1})\bigr] \le (1+\E(g_1)) \exp\bigl(A_\ga  \sum_{j=1}^t \gga_j^{1+\ga}\bigr),$$
where $A_\ga = L^2 (1+ {1\over \ga})^\ga \gk^{2(1+\ga)}.$
\endlemm

\beg{proof}   Since $\phi$ is convex and $\phi'$ is of $\ga$-H\"{o}lder continuous, by part (a) and part (d) of Proposition \ref{prop:1} we have
$$\begin{array}{ll} & \phi(yg_{t+1}(x))   \le \phi(y g_t(x))  +  \phi'(yg_t(x)) y(g_{t+1} - g_t(x)) + {L\over 1+\ga} |g_{t+1}(x) - g_t(x)|^{1+\ga} \\
& =\phi(y g_t(x))  -\gga_t  \langle \phi'(yg_t(x)) y G_x, \phi'(y_tg_t(x_t))y_t G_{x_t}\rangle + {L\over 1+\ga} |g_{t+1}(x) - g_t(x)|^{1+\ga}\\
& \le \phi(y g_t(x))  -\gga_t  \langle \phi'(yg_t(x)) y G_x, \phi'(y_tg_t(x_t))y_t G_{x_t}\rangle + {L \gk^{2(1+\ga)}\gga_t^{1+\ga}\over 1+\ga} |\phi'(y_tg_t(x_t))|^{1+\ga}\\
& \le  \phi(y g_t(x))  -\gga_t  \langle \phi'(yg_t(x)) y G_x, \phi'(y_tg_t(x_t))y_t G_{x_t}\rangle + A_\ga \gga_t^{1+\ga } |\phi(y_tg_t(x_t))|^{\ga}.
\end{array} $$
Taking expectation of both sides of the above inequality with respect to $z= (x,y)$ and samples $\{z_1,\ldots,z_t\},$ and noting that $g_t$ only depends on $\{z_1,\ldots, z_{t-1}\}$, we have
\begeqn\label{eq:recursive-ineq}\begin{array}{ll}  & \EX\bigl[\E(g_{t+1})\bigr]  \le \EX\bigl[\E(g_t)\bigr]   -\gga_t   \EX\Bigl[\|\int_\Z \phi'(yg_t(x)) y G_x d\rho(x,y)\|_G^2\Bigr]  \\ &
\quad + A_\ga \gga_t^{1+\ga} \EX\Bigl[\int_\Z|\phi(yg_t(x))|^{\ga}  d\rho(x,y) \Bigr]  \\
&\le \EX\bigl[\E(g_t)\bigr]   -\gga_t   \EX\Bigl[\|\int_\Z \phi'(yg_t(x)) y G_x d\rho(x,y)\|_G^2\Bigr]  \\ &
\quad + A_\ga \gga_t^{1+\ga} \Bigl(\EX[\int_\Z \phi(yg_t(x))  d\rho(x,y)] \Bigr)^{\ga} \\
& =  \EX\bigl[\E(g_t)\bigr]   -\gga_t   \EX\Bigl[\|\int_\Z \phi'(yg_t(x)) y G_x d\rho(x,y)\|_G^2\Bigr]+ A_\ga \gga_t^{1+\ga } \bigl(\EX[\E(g_t)]\bigr)^\ga \\
&  \le (1+ A_\ga \gga_t^{1+\ga } ) \EX\bigl[\E(g_t)\bigr]   -\gga_t   \EX\Bigl[\|\int_\Z \phi'(yg_t(x)) y G_x d\rho(x,y)\|_G^2\Bigr]+ A_\ga \gga_t^{1+\ga }.
\end{array}\endeqn
Consequently, $$\EX\bigl[\E(g_{t+1})\bigr] \le (1+ A_\ga \gga_t^{1+\ga }) \EX\bigl[\E(g_t)\bigr] + A_\ga \gga_t^{1+\ga },
$$
The above inequality implies that
  $$\begin{array}{ll}\EX\bigl[\E(g_{t+1})\bigr]  & \le \prod_{j=1}^t (1+ A_\ga \gga_t^{1+\ga } ) \E(g_1) + A_\ga \sum_{j=1}^t \prod_{k=j+1}^t (1+ A_\ga \gga_k^{1+\ga } ) \gga_j^{1+\ga } \\ & \le  \prod_{j=1}^t (1+ A_\ga \gga_t^{1+\ga } ) \E(g_1) +  \sum_{j=1}^t \bigl[\prod_{k=j}^t (1+ A_\ga \gga_k^{1+\ga } ) -\prod_{k=j+1}^t (1+ A_\ga \gga_k^{1+\ga } )  \bigr] \\
	& = \prod_{j=1}^t (1+ A_\ga \gga_t^{1+\ga } ) \E(g_1) +    \bigl[\prod_{k=1}^t (1+ A_\ga \gga_k^{1+\ga } ) - 1  \bigr]\\
	&\le (1+\E(g_1)) \exp\bigl(A_\ga \sum_{j=1}^t \gga_j^{1+\ga}\bigr).\end{array}$$
This completes the proof of the lemma.   \end{proof}
From the above lemma, we know that if $\sum_{j=1}^\infty \gga_t^{1+\ga} <\infty$ then, for any $t\in \N$, there holds \begeqn\label{eq:bound-uniform}\beg{array}{ll}\EX[\E(g_{t+1})] & \le (1+\E(g_1)) \exp\Bigl(A_\ga \sum_{j=1}^t \gga_j^{1+\ga}\Bigr) \\ & \le D_\infty:= (1+\E(g_1)) \exp\Bigl(A_\ga \sum_{j=1}^\infty \gga_j^{1+\ga}\Bigr)<\infty.
\end{array}\endeqn
One typical example of step sizes is of the form  $\gga_t= {c\over t^\gth}$ with some $\gth\in ({1\over 1+\ga}, 1).$ In this case, notice that \begeqn\label{eq:sum}\beg{array}{ll} \sum_{j=2}^t \gga_j^{1+\ga} & = c^{1+\ga} \sum_{j=1}^t j^{-\gth(1+\ga)} =  c^{1+\ga} (1+\sum_{j=2}^t j^{-\gth(1+\ga)}) \\ & \le
 c^{1+\ga} (1+\int_{1}^t s^{-\gth(1+\ga)}d s \le     {c^{1+\ga} \gth(1+\ga)\over \gth(1+\ga)-1} \le {2c^{1+\ga} \over \gth(1+\ga)-1}.\end{array}\endeqn
Hence, for any $t\in\N$,
\begeqn\label{eq:bound-2} \EX\bigl[\E(g_{t})\bigr] \le   D_{\infty}  \le \bigl(1+\E(g_1)\bigr) \exp\Bigl({ 2 A_\ga  c^{1+\ga}  \over \gth(1+\ga)-1}\Bigr). \endeqn

We now turn our attention to estimating the boundedness of $\EX\bigl[\|g_t-g_\H\|_G^2\bigr].$

\beglemm \label{eq:bound-HG} Assume that $g_\H = \arg\inf_{g\in \H_G}  \E(g)$ exists and let the learning sequence $\{g_t:  t=1,\ldots, T+1\}$ be generated by Algorithm 1. Then,
 $$ \EX\bigl[\|g_{t+1}-g_\H\|_G^2\bigr] \le \|g_\H\|_G^2  + B_\ga D_{\infty}^{ 2\ga\over 1+\ga}\sum_{j=1}^t \gga_j^2,$$
where $B_\ga :={\gk^2 (1+\ga)^2 L^{2\over 1+\ga} } \ga^{-{2\ga \over 1+\ga}}.$
\endlemm

\begin{proof}  Notice that, since $g_\H =\arg\inf_{g\in \H_G} \E(g)$, $$\int\phi'(yg_\H(x))y G_{x} d\rho(x,y)=0.$$ By the definition of $g_{t+1}$ in Algorithm 1, $\EX[\|g_{t+1} - g_\H\|_G^2]$ is therefore bounded by
\begeqn\label{eq:HK-norm-1}\begin{array}{ll}  & \EX[\|g_t-g_\H\|_G^2] - 2\gga_t \EX [\langle \phi'(y_tg_t(x_t))y_t G_{x_t}, g_t-g_\H \rangle_G]
+  \gga_t^2 \EX\bigl[\|\phi'(y_tg_t(x_t))G_{x_t}\|_G^2\bigr]\\
& \le \EX[\|g_t-g_\H\|_G^2] - 2\gga_t \EX [\langle \phi'(y_tg_t(x_t))y_t G_{x_t}, g_t-g_\H \rangle_G]+ \gga_t^2 \gk^2 \EX\bigl[|\phi'(y_tg_t(x_t))|^2\bigr]\\
& = \EX[\|g_t-g_\H\|_G^2] - 2\gga_t \EX [\langle\int [\phi'(yg_t(x))y G_{x} -\phi'(yg_\H(x))y G_{x} ]d\rho(x,y), g_t-g_\H \rangle_G]\\
& + \gga_t^2 \gk^2 \EX\bigl[|\phi'(y_tg_t(x_t))|^2\bigr]\\
& \le  \EX[\|g_t-g_\H\|_G^2]+ \gga_t^2 \gk^2 \EX\bigl[|\phi'(y_tg_t(x_t))|^2\bigr]\\
& \le \EX[\|g_t-g_\H\|_G^2]+ \gga_t^2 \gk^2 \bigl(\EX[|\phi'(y_tg_t(x_t))|^{1+\ga\over \ga}]\bigr)^{2\ga\over 1+\ga}
\end{array}\endeqn
where the second to last inequality used the fact, by part (c) of Proposition \ref{prop:1},  $$\begin{array}{ll}&\langle\int [\phi'(yg_t(x))y G_{x} -\phi'(yg_\H(x))y G_{x} ]d\rho(x,y), g_t-g_\H \rangle_G \\ & = \int [\phi'(yg_t(x)) -\phi'(yg_\H(x))]y (g_t(x)-g_\H(x)) d\rho(x,y)\ge 0.\end{array}$$
Also, by part (d) of Proposition \ref{prop:1}, we have $|\phi'(y_tg_t(x_t))|^{1+\ga\over \ga}\le
{(1+\ga)^{1+{1\over \ga}}\over \ga} L^{{1\over \ga}}\,  \phi(y_tg_t(x_t)) .$  Putting this back into (\ref{eq:HK-norm-1}),  we know from (\ref{eq:bound-uniform}) that
$$\begin{array}{ll} \EX[\|g_{t+1} - g_\H\|_G^2] &  \le \EX[\|g_t-g_\H\|_G^2]+ \gga_t^2 \gk^2 {(1+\ga)^2 L^{2\over 1+\ga} \over \ga^{2\ga\over 1+\ga }}
 \bigl[\EX(\E(g_t))\bigr]^{2\ga\over 1+\ga}\\
&\le \EX[\|g_t-g_\H\|_G^2]+ \gga_t^2 \gk^2 {(1+\ga)^2 L^{2\over 1+\ga} \over \ga^{2\ga\over 1+\ga }}
 \bigl(D_{\infty}\bigr)^{2\ga\over 1+\ga},
\end{array}$$
which directly yields the desired result. This completes the proof of the lemma. \end{proof}

Let $\bar{D}_\infty = \|g_\H\|_G^2  + B_\ga D_\infty^{2\ga\over 1+\ga} \sum_{j=1}^\infty \gga_j^2 .$ Then, if the step sizes are in the form of $\gga_t = {c\over t^{\gth}}$ with $\gth \in ({1\over 1+\ga}, 1)$, then, by (\ref{eq:bound-2}),
\begeqn\label{eq:HK-bound-2}\beg{array}{ll} \EX\bigl[\|g_{t}-g_\H\|_G^2\bigr] & \le \|g_\H\|_G^2  +  c^2 B_\ga D_\infty^{2\ga\over 1+\ga}\sum_{j=1}^{t-1} j^{-2\gth} \\
& \le \bar{D}_\infty\le  \|g_\H\|_G^2  +  {2\gth c^2 B_\ga D_\infty^{2\ga\over 1+\ga}\over 2\gth-1} \end{array}\endeqn

We are now in a position to prove the main theorems for Algorithm 1.

\noindent {\bf Proof of Theorem \ref{thm:conv-1}.}  By (\ref{eq:recursive-ineq}) and (\ref{eq:bound-uniform}), we have
\begeqn\label{eq:recursive}\EX\bigl[\E(g_{t+1})\bigr]  \le \EX\bigl[\E(g_t)\bigr]   -\gga_t   \EX\Bigl[\|\int_\Z \phi'(yg_t(x)) y G_x d\rho(x,y)\|_G^2\Bigr]+ A_\ga (1+D_\infty) \gga_t^{1+\ga }.\endeqn
The above inequality implies that
$$ \EX\bigl[\E(g_{t+1})\bigr] \le \EX\bigl[\E(g_{t})\bigr] + A_\ga (1+D_\infty) \gga_t^{1+\ga }.$$
Consequently, for any fixed $t\le T$,
$$  \EX\bigl[\E(g_{T+1})\bigr] \le \EX\bigl[\E(g_{t})\bigr] + A_\ga (1+D_\infty) \sum_{j =t}^\infty \gga_t^{1+\ga }.$$
This means that $\overline{\lim}_{T\to \infty}\EX\bigl[\E(g_{T+1})\bigr] \le \EX\bigl[\E(g_{t})\bigr]+ A_\ga (1+D_\infty) \sum_{j =t}^\infty \gga_t^{1+\ga },$ which also implies, since $\sum_{j=1}^\infty \gga_t^{1+\ga} <\infty,$ that
$$\overline{\lim}_{T\to \infty}\EX\bigl[\E(g_{T+1})\bigr] \le \underline{\lim}_{t\to \infty} [\EX\bigl[\E(g_{t})\bigr]+ A_\ga (1+D_\infty) \sum_{j =t}^\infty \gga_t^{1+\ga } = \underline{\lim}_{t\to \infty}\EX\bigl[\E(g_{t})\bigr].$$
Hence, $\gep:= \lim_{t\to \infty}\EX\bigl[\E(g_{t})\bigr]$ exists and, apparently, $ \inf_{g\in \H_G} \E(g) \le \gep \le D_{\infty}<\infty$ where the last inequality follows from equation (\ref{eq:bound-uniform}).  This completes the proof for the first part of the theorem.

Now it remains to prove, if we further assume that $g_\H = \arg\inf_{g\in \H_G} \E(g)$ exists and $\sum_{j=1}^\infty \gga_j = \infty,$ that $\gep= \inf_{g\in \H_G} \E(g)$. Let us assume, on the contrary, that $\gep_1=\gep - \inf_{g\in \H_G} \E(g)>0.$  Let $R_t:=\EX\bigl[\E(g_{t})\bigr] - \inf_{g\in \H_G} \E(g) $ for any $t\in \N.$ In this case, there exists $t_1$ such that, for any $t\ge t_1$, $R_{t} \ge {\gep_1\over 2}.$ However, from (\ref{eq:recursive}), we know that
\begeqn\label{eq:recursive-21} R_{t+1} \le R_t -\gga_t   \EX\Bigl[\bigl\|\int_\Z \phi'(yg_t(x)) y G_x d\rho(x,y)\bigr\|_G^2\Bigr]
+  A_\ga (1+D_\infty)   \gga_t^{1+\ga }.\endeqn
By the convexity of $\phi$, $$\begin{array}{ll} \E(g_t) - \E(g_\H) & \le \int_\Z \phi'(yg_t(x)) y (g_t(x) -g_\H(x)) d\rho(x,y) \\
& = \langle\int_\Z \phi'(yg_t(x)) y G_x d\rho(x,y), g_t-g_\H\rangle_G\\
& \le \bigl[\|\int_\Z \phi'(yg_t(x)) y G_x d\rho(x,y)\|^2\bigr]^{1\over 2}\|g_t-g_\H\|_G.
\end{array}$$
 Also, observe that $ \bar{D}_\infty = \|g_\H\|_G^2  + B_\ga D_\infty^{2\ga\over 1+\ga} \sum_{j=1}^\infty \gga_j^2 < \infty,$ since $\sum_{j=1}^\infty \gga_j^{1+\ga}<\infty$ and $\ga\le 1.$ This implies that
$$ \EX\Bigl[\bigl\|\int_\Z \phi'(yg_t(x)) y G_x d\rho(x,y)\bigr\|_G^2\Bigr] \ge {R_{t}^2 \over \EX\bigl[\|g_t-g_\H\|_G^2\bigr]}\ge {R_{t}^2 \over \bar{D}_\infty}. $$
Putting this back into (\ref{eq:recursive-21})  yields that
\begeqn\label{eq:recursive-final} R_{t+1} \le R_t -\gga_t  {  R_{t}^2 / \bar{D}_\infty}
+  A_\ga (1+D_\infty)  \gga_t^{1+\ga }. \endeqn
This means that $$\begin{array}{ll}\overline{\lim}_{T\to \infty}\sum_{t=1}^T \gga_t  {  R_{t}^2 / \bar{D}_\infty} & \le R_1 + A_\ga (1+D_\infty)  \sum_{t=1}^T \gga_t^{1+\ga }\\ &  \le R_1 + A_\ga (1+D_\infty)  \sum_{t=1}^\infty \gga_t^{1+\ga } <\infty.\end{array}$$
However, $\sum_{t=1}^T \gga_t  {  R_{t}^2 / \bar{D}_\infty} \ge {    \gep_1^2 \over 4\bar{D}_\infty}  \sum_{t=t_1}^T \gga_t, $ which implies, by the assumption that $\sum_{t=1}^\infty\gga_t = \infty,$ that $$\overline{\lim}_{T\to \infty} \sum_{t=1}^T \gga_t  {  R_{t}^2 / \bar{D}_\infty}  \ge {    \gep_1^2 \over 4\bar{D}_\infty} {    \gep_1^2 \over 4\bar{D}_\infty} \sum_{t=t_1}^\infty \gga_t =\infty. $$
This leads to a contradiction. Hence, $\gep_1 =\displaystyle\lim_{t\to \infty}R_{t}=0.$   This completes the proof the theorem.
\hfill  $\Box$

We now turn our attention to proving Theorem \ref{thm:rate-1} by an  induction based on the recursive inequality (\ref{eq:recursive-final}).

\noindent {\bf Proof of Theorem \ref{thm:rate-1}.}    We prove the theorem from the recursive inequality (\ref{eq:recursive-final}).   Since $\gga_t = {c\over t^{\gth}}$ with some $\gth \in ({1\over 1+\ga}, 1)$, inequalities (\ref{eq:bound-2}) and (\ref{eq:HK-bound-2}) hold true. Let  $\gb = \min({\ga\gth\over 2}, 1-\gth), $
and choose $$D  = \max\Bigl\{{D}_\infty, \bigl({2c\over \bar{D}_\infty}\bigr)^{\min({\ga\over 2}, {1-\gth\over \gth})}\bigl(2^\gb{D}_\infty\bigr)^{\min(1+{\ga\over 2}), {1\over \gth})}, {\bar{D}_\infty\over c} + \sqrt{A_\ga (1+{D}_\infty)\bar{D}_\infty \over c}  \Bigr\}.
$$ Denote
 $$t_0 =    \left\lfloor 2  ({2 c D \over \bar{D}_\infty})^{1\over \gth +\gb}\right\rfloor .$$ By the definition of $D$ and $\gb$, we know that $D\ge {\bar{D}_\infty\over c}$ and $0<\gth+\gb\le 1$ which further implies that  $t_0\ge 4.$
Since  $$D  \ge \max\Bigl\{{D}_\infty, \bigl({2c\over \bar{D}_\infty}\bigr)^{\min({\ga\over 2}, {1-\gth\over \gth})}\bigl(2^\gb{D}_\infty\bigr)^{\min(1+{\ga\over 2}, {1\over \gth})}\Bigr\},$$ we have
$$  \EX[\E(g_t) -\E(g_\H)]\le {D}_\infty  \le {D \over t_0^\gb}\le {D \over t^\gb}, ~~\forall t\le t_0.$$
Now we assume that $R_t \le {D \over t^\gb}$ for some $t\in \N$ and $t\ge t_0$ and we are going to prove that $R_{t+1}\le {D \over (t+1)^{\gb}}$ by induction.

To this end, let $F(x) :=  x -\gga_t  {  x^2 / \bar{D}_\infty}$ and notice that $F$ is increasing when $x\in (0, {\bar{D}_\infty t^\gth\over 2c}].$   Observe that $t\ge t_0 \ge \bigl({2 c D \over \bar{D}_\infty}\bigr)^{1\over \gth +\gb}$ which implies that
$  {D \over t^\gb} \in (0, {\bar{D}_\infty t^\gth \over 2 c}).$  Combining this with (\ref{eq:recursive-final}) and the induction assumption $R_t \le {D \over t^\gb}$ (i.e. $R_t \in (0,{\bar{D}_\infty t^\gth \over 2 c})$), we have
\begeqn\label{eq:inequ-inter}\begin{array}{ll} R_{t+1} &\le F(R_t) +  A_\ga (1+D_\infty)  \gga_t^{1+\ga } \le F\bigl({D\over t^\gb}\bigr) + A_\ga (1+D_\infty)  \gga_t^{1+\ga }\\
& \le {D\over t^\gb} \Bigl[ 1- \bigl( {cD \over \bar{D}_\infty} -  {A_\ga (1+{D}_\infty) \over D} t^{2\gb-\gth \ga}\bigr) t^{-\gth-\gb}\Bigr]\\
& \le {D\over t^\gb} \Bigl[ 1- \bigl( {cD \over \bar{D}_\infty} -  {A_\ga (1+{D}_\infty) \over D} \bigr) t^{-\gth-\gb}\Bigr],
\end{array}\endeqn
where the last inequality used that fact $ 2\gb-\gth \ga\le 0.$ By the definition of $D$, $D  \ge {\bar{D}_\infty\over c} + \sqrt{A_\ga (1+{D}_\infty)\bar{D}_\infty \over c}$ which implies that ${cD \over \bar{D}_\infty} -  {A_\ga (1+{D}_\infty) \over D}\ge 1.$  Putting this back into (\ref{eq:inequ-inter}) yields that
$$R_{t+1} \le {D\over t^\gb} \Bigl[ 1-   t^{-\gth-\gb}\Bigr]\le {D\over t^\gb} \Bigl[ 1-   t^{-1}\Bigr] =  {D\over t^\gb} \bigl({t-1\over t}\bigr) \le {D\over t^\gb} \bigl({t\over t+1}\bigr)^\gb = {D\over (t+1)^\gb},$$
where the second inequality used the fact that $\gth+\gb\le 1.$
This completes the proof of the theorem.  \hfill $\Box$

\subsection{Proofs for the Convergence of Algorithm 2}
In this subsection,  we prove the main  theorems related to Algorithm 2. The main idea is to   derive a recursive inequality on the sequence $\{R_t := \EX[\wtE(f_t) - \wtE(f_\H)]:  1\le t\le T+1\}$ (i.e. the relationship between $R_{t+1}$ and $R_t$), and then apply a smart induction on this  inequality. To do this,   let us establish some useful lemmas.  Denote $\gtk=\sup_{x,\bar{x}\in \X\times \X}
\sqrt{K((x,\bar{x}), (x,\bar{x}))}.$

\beglemm\label{lemm:bound-2} Assume $\phi$ is $1$-activating and $f_\H = \arg\inf_{f\in \H_K} \wtE(f)$ exists. Let $\{f_t: t=1,\dots, T+1\}$ be generated by Algorithm 2. Then
$$\EX\bigl[\|f_{t+1}-f_\H\|^2_K\bigr] \le\Bigl[\|f_\H\|_K^2 + \gs_\H^2(4 + \ln t)\Bigr] \exp\bigl( (1+32\gtk^4 L^2 )\sum_{j=2}^t \gga_j^2\bigr)   , $$
where $\gs_\H^2 =  \dint_\Z \int_\Z  \bigl\|\phi'((y-\widetilde{y})f_\H(x,\wtx)) K_{(x,\wtx)}\bigr      \|_K^2 d\rho(x,y)d\rho(\wtx,\widetilde{y}).$
\endlemm

\begin{proof} $ \EX\bigl[\|f_{t+1} - f_\H\|_K^2\bigr]$ is bounded by
\begeqn\label{eq:hknorm-inter1}\hspace*{-0.2cm}\beg{array}{ll}  & \EX[\|f_t - f_\H\|_K^2]  + {\gga_t^2\over (t-1)^2} \EX\bigl[\bigl\| \dsum_{j=1}^{t-1} \phi'((y_t-y_j)f_t(x_t,x_j)) (y_t-y_j)K_{(x_t,x_j)}\bigr\|^2\bigr] \\
&  - {2\gga_t \over t-1} \EX\bigl[\dsum_{j=1}^{t-1} \phi'((y_t-y_j)f_t(x_t,x_j)) (y_t-y_j)(f_t(x_t,x_j) -f_\H(x_t,x_j))\bigr].\end{array}\endeqn
Noting that $\dint_\Z \int_\Z  \phi'((y-\widetilde{y})f_\H(x,\wtx)) K_{(x,\wtx)}       d\rho(x,y)d\rho(\wtx,\widetilde{y}) =0$ , we have
$$\beg{array}{ll} & -\EX[\sum_{j=1}^{t-1} \phi'((y_t-y_j)f_t(x_t,x_j)) (y_t-y_j)(f_t(x_t,x_j) -f_\H(x_t,x_j))]\\
& = -\EX[\dsum_{j=1}^{t-1} [\phi'((y_t-y_j)f_t(x_t,x_j))-\phi'((y_t-y_j)f_\H(x_t,x_j))] (y_t-y_j)(f_t(x_t,x_j) -f_\H(x_t,x_j))]\\
&  -\EX[ \dsum_{j=1}^{t-1} \phi'((y_t-y_j)f_\H(x_t,x_j)) (y_t-y_j)(f_t(x_t,x_j) -f_\H(x_t,x_j))]\\
& \le \EX[\dsum_{j=1}^{t-1} \phi'((y_t-y_j)f_\H(x_t,x_j)) (y_t-y_j)(f_\H(x_t,x_j) -f_t(x_t,x_j))]\\
& = \EX[ \langle \dsum_{j=1}^{t-1} \phi'((y-y_j)f_\H(x,x_j)) (y-y_j)K_{(x,x_j)}, f_\H-f_t\rangle_K]\\
& \le \sqrt{t-1}  (\EX[\|f_t-f_\H\|_K^2])^{1\over 2}\gs_\H  \le {1\over 2} \bigl[(t-1)\gga_t \EX[\|f_t-f_\H\|_K^2]  + {\gs_\H^2 \over \gga_t}\bigr].
\end{array}$$
Also, $\EX\bigl[\bigl\| \sum_{j=1}^{t-1} \phi'((y_t-y_j)f_t(x_t,x_j)) (y_t-y_j)K_{(x_t,x_j)}\bigr\|^2\bigr]$ can be bounded by
$$\begin{array}{ll}  & 2 \EX\bigl[\bigl\| \dsum_{j=1}^{t-1} (\phi'((y_t-y_j)f_t(x_t,x_j)) -\phi'((y_t-y_j)f_\H(x_t,x_j)))(y_t-y_j)K_{(x_t,x_j)}\bigr\|^2\bigr] \\
 & +  2  \EX\bigl[\bigl\| \dsum_{j=1}^{t-1} \phi'((y_t-y_j)f_\H(x_t,x_j)) (y_t-y_j)K_{(x_t,x_j)}\bigr\|^2\bigr] \\
 & \le 32\gtk^4 L^2 (t-1)^2 \|f_t-f_\H\|^2_K     + 2 (t-1) \gs_\H^2.  \end{array}$$
Putting these two estimates into (\ref{eq:hknorm-inter1}), we have
$$\EX\bigl[\|f_{t+1} - f_\H\|_K^2\bigr] \le (1+ (32\gtk^4 L^2 +1)\gga_t^2) \EX\bigl[\|f_{t} - f_\H\|_K^2\bigr] + {(2 \gga_t^2 + 1)\gs_\H^2\over t-1} .$$
Therefore,
$$\begin{array}{ll}& \EX\bigl[\|f_{t+1} - f_\H\|_K^2\bigr]  \le \prod_{j=2}^ t (1+ (32\gtk^4 L^2 +1)\gga_j^2) \|f_\H\|_K^2  \\ & + \gs_\H^2\sum_{j=2}^t
 \prod_{k=j+1}^t (1+ (32\gtk^4 L^2 +1)\gga_k^2)\bigl[2\gga_j^2 + {1\over j-1}  \bigr]\\
& \le \exp\bigl( (32\gtk^4 L^2 +1)\sum_{j=2}^t \gga_j^2\bigr) \|f_\H\|_K^2 \\ & + {2\gs_\H^2\over 32\gtk^4 L^2 +1}\sum_{j=2}^t
 \Bigl[\prod_{k=j}^t (1+ (32\gtk^4 L^2 +1)\gga_k^2)-  \prod_{k=j+1}^t (1+ (32\gtk^4 L^2 +1)\gga_k^2)\Bigr] \\
& + \gs_\H^2\sum_{j=2}^t \prod_{k=j+1}^t (1+ (32\gtk^4 L^2 +1)\gga_k^2){2\over j-1}\\
& \le  \exp\bigl( (1+32\gtk^4 L^2 )\sum_{j=2}^t \gga_j^2\bigr) \bigl[\|f_\H\|_K^2 +\gs_\H^2(4+ \ln t)\bigr].
\end{array} $$
This completes the proof of the lemma.
\end{proof}

From the above lemma, we know if $\gga_t = {c\over t^{\gth} }$ with some $\gth \in ({1\over 2}, 1).$ Then,
\begeqn\label{eq:Et-est} \EX\bigl[\|f_{t}-f_\H\|^2_K\bigr] \le  {E_t}:=\exp\bigl( {(1+32 \gtk^4 L^2)c^2 \over 2\gth-1}\bigr) [\|f_\H\|_K^2 + \gs_\H^2({4} +  \ln t )]\endeqn

The next lemma estimates the boundedness of the learning sequence under the RKHS norm.

\beglemm \label{lemm:HK-bound-2}  Let $\phi$ be $1$-activating   and $\{f_t: t=1,\ldots, T+1\}$ be given by Algorithm 2. If $\gga_t \gtk^2 \le {1\over 4 L}$ for any $t\in \N$ then
 $$\|f_{t+1}\|_K \le \wtD_t=   C_\phi \sqrt{\sum_{j=2}^t \gga_j},$$
where $C_\phi = \sqrt{L} s_0$ if there exists  $s_0\in \R$ such that $\phi'(s_0)=0$, and $C_\phi = \sqrt{2\phi(0) + {2(\phi'(0))^2 \over L}}$ otherwise. \endlemm

\begin{proof} Write
$$\begin{array}{ll} \|f_{t+1}\|_K^2 & \le  \|f_t\|_K^2   +  {\gga_t^2\over (t-1)^2}\|  \sum_{j=1}^{t-1}\phi'((y_t-y_j)f_t(x_t,x_j)) (y_t-y_j)K_{(x_t,x_j)}\|^2\\
& - {2\gga_t \over t-1} \sum_{j=1}^{t-1}\phi'((y_t-y_j)f_t(x_t,x_j)) (y_t-y_j)f_t(x_t,x_j)\\
& \le \|f_t\|_K^2   +  {\gga_t\over t-1} \sum_{j=1}^{t-1} \Bigl[ 4 \gtk^2 \gga_t |\phi'((y_t-y_j)f_t(x_t,x_j))|^2 \\
& - 2 \phi'((y_t-y_j)f_t(x_t,x_j)) (y_t-y_j)f_t(x_t,x_j)\Bigr]\\
& \le  \|f_t\|_K^2   +\gga_t\sup_{s\in \R}   \bigl[ 4 (\phi'(s))^2  \gga_t \gtk^2   - 2\phi'(s)s\bigr].
\end{array}$$
Therefore, the desired result follows directly from the following claim:
\begeqn\label{eq:claim}\sup_{s\in \R}   \bigl[ 4 (\phi'(s))^2  \gga_t \gk^2   - 2\phi'(s)s\bigr]  \le C^2_\phi,\qquad \quad  \hbox{ if } \gga_t \gtk^2\le {1\over 4 L}.\endeqn
To prove (\ref{eq:claim}), we discuss the following two cases.

\noindent {\em Case 1:  $\phi'(s)\le 0$ for any $s\in \R$.} Firstly, consider $s\ge 0.$ By the convexity of $\phi$, $-s \phi'(s)\le \phi(0) - \phi(s) \le \phi(0).$  In addition, $\phi'(0) \le \phi'(s)\le 0.$   Hence, for $s\ge 0$, there holds
\begeqn\label{eq:s0+} 4 (\phi'(s))^2  \gga_t \gk^2   - 2\phi'(s)s \le 4 (\phi'(s))^2 \gga_t \gtk^2 +2\phi(0) \le {(\phi'(0))^2 \over L} +2\phi(0). \endeqn

Secondly, consider $s<0$ which implies $s\phi'(0) >0.$ Since $\phi'(\cdot)$ is Lipschitz continuous,   part (c) of Proposition \ref{prop:1} implies that
$(\phi'(s)-\phi'(0)) s \ge {(\phi'(s)-\phi'(0))^2\over L} = {(|\phi'(s)| - |\phi'(0)|)^2\over L}.$ Therefore, for $s<0$, we have
\begeqn\label{eq:s0-}\begin{array}{ll} 4 (\phi'(s))^2  \gga_t \gtk^2   - 2\phi'(s)s & \le   4 (\phi'(s))^2  \gga_t \gtk^2   - 2(\phi'(s)-\phi'(0)) s \\ & \le
4 (\phi'(s))^2  \gga_t \gtk^2   -  {2(|\phi'(s)| - |\phi'(0)|)^2\over L}\\ &   \le {(\phi'(s))^2 \over L }   - {2(|\phi'(s)| - |\phi'(0)|)^2\over L} \\
&  = -{1\over L} {(|\phi'(s)| - 2|\phi'(0)|)^2} + {2(\phi'(0))^2\over L} \le {2(\phi'(0))^2\over L}. \end{array}
\endeqn
Combining the above estimates (\ref{eq:s0+}) and (\ref{eq:s0-}) yields that
$$  \sup_{s\in \R } \bigl[  4 (\phi'(s))^2  \gga_t \gtk^2   - 2\phi'(s)s\bigr] \le  2\phi(0) + {2(\phi'(0))^2\over L}. $$

\noindent {\em Case 2: $\phi'(s_1)>0$ for some $s_1\in \R$.} Since $\phi'$ is increasing and $\phi'(0)<0$ by assumption,  therefore $s_1$ must be positive and there exists $s_0>0$ such that $\phi'(s_0)=0.$  Hence, by part (b) of Proposition \ref{prop:1}, we have
  $$\beg{array}{ll} 4 (\phi'(s))^2  \gga_t \gtk^2   - 2\phi'(s)s &  = 4 (\phi'(s))^2  \gga_t \gtk^2   - 2(\phi'(s)-\phi'(s_0)) (s-s_0) -2s_0 \phi'(s)\\
	& \le 4 (\phi'(s))^2  \gga_t \gtk^2   - {2\over L}(\phi'(s)-\phi'(s_0))^2-2s_0 \phi'(s)\\
	& =   (4\gga_t \gtk^2 - {2\over L}) (\phi'(s))^2   -2s_0 \phi'(s)\\
	& \le - {1\over L}(\phi'(s))^2   -2s_0 \phi'(s) =  -{1\over L} (\phi'(s)+Ls_0)^2 + L (s_0)^2,
	\end{array}  $$
which implies that
$$  \sup_{s\in \R } \bigl[ 4(\phi'(s))^2  \gga_t \gtk^2   - 2\phi'(s)s\bigr] \le  L (s_0)^2.  $$
Combining the estimates in the above two cases yields (\ref{eq:claim}).  This completes the proof of the lemma.
\end{proof}

From the above lemma, we know that if $\gga_t = {c \over t^\gth}$ with $\gth\in (0,1)$ then
\begeqn\label{eq:DT-gth} \|f_t\|_K \le C_\phi \sqrt{\sum_{j=2}^{t-1} \gga_j} \le {\sqrt{c} C_\phi \over \sqrt{1-\gth}}\, t^{1-\gth\over 2}.  \endeqn

Our analysis for Algorithm 2 also needs the concept of Rademacher averages \cite{BM}. Let $\mathcal{F}$ be a class of uniformly bounded functions. The (empirical) Rademacher average $R_n(\mathcal{F})$ over $\mathcal{F}$ is defined by   $$R_n(\mathcal{F}):=\mathbb{E}_{\sigma}\left[\sup_{f\in F}\frac{1}{n}\sum_{j=1}^n\sigma_i f(z_j)\right],$$
where $\{z_j: j=1,2,\ldots,n\}$ are independent random variables distributed according to some probability measure and $\{\sigma_j: j=1,2,\ldots,n\}$ are independent Rademacher random variables, that is, $P(\sigma_j= 1)=P(\sigma_j=-1)=\frac{1}{2}.$ Another useful complexity to describe the capacity of $\F$ is the Gaussian average  which is defined by $$G_n(\mathcal{F}):=\mathbb{E}_{\sigma}\left[\sup_{f\in F}\frac{1}{n}\sum_{j=1}^n g_j f(z_j)\right],$$
where $\{g_j: j=1,2,\ldots,n\}$ are independent Gaussian $\mathcal{N}(0,1)$ random variables. The following inequality (e.g. \cite[Remark 2.26]{Shahar}) describes the relationship between the above complexity averages:
\begeqn\label{eq:rad-gau-average} {\rho G_n(F)\over \ln n }\le R_n(F) \le \mu \, G_n(F). \endeqn
Here, $\mu>0$ and $\rho>0$ are absolute constants independent of $F$ and $n$.

We begin with stating the well-known comparison principles for Gaussian process (e.g. \cite{vitale2000some}) which will be used to prove a useful property of Gaussian averages.

\begin{lemma}
  Let $\{X_\theta:\theta\in \Theta\}$ and $\{Y_\theta:\theta\in\Theta\}$ be two zero-mean Gaussian process indexed by the same countable set $\Theta$ and suppose that $$\EX_g[(Y_\theta-Y_{\bar{\theta}})^2]\leq \EX_g[(X_\theta-X_{\bar{\theta}})^2],\quad\forall \theta,\bar{\theta}\in\Theta.$$Then, $$\EX_g[\sup_\theta Y_\theta]\leq\EX_g[\sup_\theta X_\theta].$$
\end{lemma}
We now can derive the following property related to the Gaussian average.
\beglemm\label{lemm:Rad-property} Let $F_j(\gth)$ be a set of functions indexed by parameters $\gth = (\gth_1,\gth_2)\in \gTh_1\times \gTh_2$, $H_j(\gth_1)$, and $J_j(\gth_2)$ be a set of functions indexed, respectively, by parameter $\gth_1 \in \gTh_1,$ and $\gth_2\in \gTh_2.$ Assume, for any $\gth = (\gth_1,\gth_2), \bar{\gth} = (\bar{\gth}_1, \bar{\gth}_2)\in \Theta_1 \times \gTh_2$, that
$$|F_j(\gth) - F_j(\bar{\gth})|^2 \le  |H_j(\gth_1) - H_j(\bar{\gth}_1)|^2 + |J_j(\gth_2)-J_j(\bar{\gth}_2)|^2.$$ Then,  
$$\EX_g[\sup_{(\gth_1,\gth_2)\in \gTh_1\times \gTh_2} \sum_{i=1}^n g_i F_i(\gth)]  \le  \EX_g [\sup_{\gth_1\in \gTh_1}\sum_{j=1}^n g_j H_j(\gth_1)] + \EX_g [\sup_{\gth_2\in \gTh_2}\sum_{j=1}^n g_j J_j(\gth_2)].
$$
\endlemm

\begin{proof}
  Let $g_1,\ldots,g_{2n}$ be $2n$ independent $\mathcal{N}(0,1)$ Gaussian variables. Introduce two Gaussian processes:
$$    X_\theta=\sum_{j=1}^ng_j F_j(\theta) \quad \hbox{and} \quad Y_\theta=\sum_{j=1}^n[g_j H_j(\theta_1)+g_{n+j}J_j(\theta_2)].$$
Then, $\EX_g[(X_\theta-X_{\bar{\theta}})^2]=\sum_{j=1}^n[F_j(\theta)-F_j(\bar{\theta})]^2,$ and $\EX_g[(Y_\theta-Y_{\bar{\theta}})^2]=\sum_{j=1}^n\big[(H_j(\theta_1)-H_j(\bar{\theta}_1))^2+(J_j(\theta_2)-J_j(\bar{\theta}_2))^2\big].$
  According to Lemma \ref{lemm:Rad-property},  we have
  $$\begin{array}{ll}
    &\EX_g[\dsup_{\theta\in\Theta}\sum_{j=1}^n g_jF_j(\theta)]\leq\EX_g[\dsup_{\theta\in\Theta}(\sum_{j=1}^ng_jH_j(\theta_1)+\sum_{j=1}^ng_{n+j}J_j(\theta_2))]\\
    &\leq\EX_g[\dsup_{\theta_1\in\Theta_1}\sum_{j=1}^ng_jH_j(\theta_1)]+\EX_g[\dsup_{\theta_2\in\Theta_2}\sum_{j=1}^ng_{n+j}J_j(\theta_2)]\\
    & =\EX_g[\sup_{\theta_1\in\Theta_1}\sum_{j=1}^ng_jH_j(\theta_1)]+\EX_g[\dsup_{\theta_2\in\Theta_2}\sum_{j=1}^ng_{j}J_j(\theta_2)].
  \end{array}$$ This completes the proof of the lemma.
\end{proof}

Denote
\begeqn\label{eq:Dtphi} M_t^\phi = \sup_{|t|\le 2\gtk \wtD_t} |\phi'(t)|.\endeqn
We also need to bound the following term defined by
$$\gD_t:=  \nabla \wtE(f_t)-{1\over t-1} \sum_{j=1}^{t-1}\int_\Z\phi'((y-y_j)f_t(x,x_j))(y-y_j)K_{(x, x_j)}d\rho(x,y),$$
where $ \nabla \wtE(f_t)$ denotes the functional derivative of $\wtE(\cdot)$ at $f_t$ given by
$$\nabla \wtE(f_t) =\iint_{\Z\times \Z} \phi'((y-\wty)(f_t(x,\wtx)))(y-\wty)K_{(x, \wtx)}d\rho(z)d\rho(\widetilde{z}).$$
Using Lemma \ref{lemm:Rad-property}, we can prove the following estimation.

\begin{lemma}\label{lemm:gd-estimate} Let $\phi$ be $1$-activating, and $\{f_t: t=1,\ldots,T+1\}$ be given by Algorithm 2.  If $\gga_t \gtk^2 \le {1\over 4 L}$ then, for any $t\ge 2$,
$$\EX[\|\gD_t\|_K] \le {8\sqrt{2}\mu(L \gtk \wtD_t  +M^\phi_t)\gtk\over \sqrt{t-1}}.$$
\endlemm

\begin{proof} For any fixed $\wtz = (\wtx,\wty)$, letting $\xi_{f,g}(\wtz,z_j) = \phi'((\wty-y_j)f(\wtx,x_j))(\wty-y_j)g{(\wtx, x_j)}.$ Since $\gga_t \gtk^2 \le {1\over 4 L}$, by Lemma \ref{lemm:HK-bound-2}, $\|f_t\|_K\le \wtD_t.$ Notice  \begeqn\label{eq:gd-t}\beg{array}{ll} \|\gD_t\|_K &  = \dsup_{\|g\|_K\le 1} \bigl[ \iint \phi'((y-\wty)(f_t(x,\wtx)))(y-\wty)g{(x, \wtx)}d\rho(z)d\rho(\widetilde{z})  \\ & \qquad \qquad - {1\over t-1} \sum_{j=1}^{t-1}\int_\Z\phi'((y-y_j)f_t(x,x_j))(y-y_j)g{(x, x_j)}d\rho(x,y)\bigr] \\
&\le  \dsup_{\|f\|_K\le \wtD_t \atop \|g\|_K\le 1} \bigl[ \iint \phi'((y-\wty)(f(x,\wtx)))(y-\wty)g{(x, \wtx)}d\rho(z)d\rho(\widetilde{z})  \\ & \qquad \qquad - {1\over t-1} \sum_{j=1}^{t-1}\int_\Z\phi'((y-y_j)f(x,x_j))(y-y_j)g{(x, x_j)}d\rho(x,y)\bigr]\\
&  = \dint_\Z \dsup_{\|f\|_K\le \wtD_t \atop \|g\|_K\le 1}\bigl[\EX_z \xi_{f,g}(\wtz,z) -  {1\over t-1} \sum_{j=1}^{t-1} \xi_{f,g}(\wtz,z_j)\bigr]d\rho(\wtz).
\end{array}\endeqn
For any fixed $\wtz = (\wtx,\wty)$,  by the standard symmetrization technique \cite{Bartlett}, from the above inequality we have
\begeqn\label{eq:symm}\beg{array}{ll}& \dsup_{\|f\|_K\le \wtD_t \atop \|g\|_K\le 1}\bigl[\EX_z \xi_{f,g}(\wtz, z) -  {1\over t-1} \sum_{j=1}^{t-1} \xi_{f,g}(\wtz,z_j)\bigr] \\ & \le 2 \EX_z \EX_{\gs}\dsup_{\|f\|\le\wtD_t \atop \|g\|_K\le 1} {1\over t-1} \sum_{j=1}^{t-1} \gs_j \xi_{f,g}(\wtz,z_j)\\
&\le  2 \mu \EX_z \EX_{g}\dsup_{\|f\|\le\wtD_t \atop \|g\|_K\le 1} {1\over t-1} \sum_{j=1}^{t-1} g_j \xi_{f,g}(\wtz,z_j).\end{array} \endeqn
Let $\gTh_1 = \{f\in \H_K: \|f\|_K\le \wtD_t \}$ and $\gTh_2 = \{g\in \H_K:  \|g\|_K\le 1\}.$ Then, for any $f,\bar{f}\in \gTh_1$ and $g,\bar{g}\in \gTh_2$, there holds
$$\begin{array}{ll}| \xi_{f,g}(\wtz,z) - \xi_{\bar{f},\bar{g}}(\wtz,z)|^2 & \le (4\sqrt{2} L \gtk |f(x,x_j) - \bar{f}(x,x_j)|)^2 +  (2 \sqrt{2}M_t^\phi |g(x,x_j)- \bar{g}(x,x_j)|)^2 \end{array}$$
Applying Lemma \ref{lemm:Rad-property} with $F_i (\gth) = \xi_{f,g}(\wtz,z)$ with $\gth_1 = f $, $\gth_2 = g$, $H_j(\gth_1) = 4\sqrt{2} L \gtk f(x,x_j)$, and $J_j(\gth_2) = 2\sqrt{2} D^\phi_t g(x,x_j)$ yields that
$$\begin{array}{ll} & \EX_g\dsup_{\|f\|\le \wtD_t \atop \|g\|_K\le 1} {1\over t-1} \sum_{j=1}^{t-1} g_j \xi_{f,g}(\wtz,z_j) \\ &
\le   4 \sqrt{2}L \gtk \EX_g  [\dsup_{\|f\|\le \wtD_t}{1\over t-1}\sum_{j=1}^{t-1} g_j  f(x,x_j)] + 2\sqrt{2}M_t^\phi \EX_g [\dsup_{\|g\|_K\le 1}{1\over t-1}\sum_{j=1}^{t-1} g_j  g(x,x_j) ] \\
&  =   4 \sqrt{2}L \gtk \,\EX_g  \dsup_{\|f\|\le \wtD_t}\langle {1\over t-1}\sum_{j=1}^{t-1} g_j  K_{(x,x_j)},  f\rangle_K + 2\sqrt{2} M_t^\phi \, \EX_g \dsup_{\|g\|_K\le 1}\langle {1\over t-1}\sum_{j=1}^{t-1} g_j  K_{(x,x_j)},g\rangle_K \\
& \le 4 \sqrt{2} L \gtk \wtD_t  \EX_g \|{1\over t-1}\sum_{j=1}^{t-1} g_j  K_{(x,x_j)}\|_K + 2 \sqrt{2}M_t^\phi \, \EX_g \|{1\over t-1}\sum_{j=1}^{t-1} g_j  K_{(x,x_j)}\|_K \\
&\le 4 \sqrt{2} L \gtk \wtD_t  \bigl(\EX_g \|{1\over t-1}\sum_{j=1}^{t-1} g_j  K_{(x,x_j)}\|^2_K\bigr)^{1/2} + 2 \sqrt{2}M_t^\phi \, \bigl(\EX_g \|{1\over t-1}\sum_{j=1}^{t-1} g_j  K_{(x,x_j)}\|_K^2\bigr)^{1/2}\\
& \le {4\sqrt{2}(L \gtk \wtD_t  +M^\phi_t)\gtk\over \sqrt{t-1}}.
\end{array}$$
Putting the above estimation, (\ref{eq:gd-t}), and (\ref{eq:symm}) together yields the desired result.
\end{proof}

Denote, for any $t\in \N$,  by $R_t = \EX[\wtE(f_t) - \wtE(f_\H)].$ We derive the following recursive inequality for $R_t$ which is critical for proving Theorem \ref{thm:rate-2}.

\beglemm\label{lemm:recursive} Let $\phi$ be an $1$-activating loss,  $\{f_t: t=1,\ldots,T+1\}$ be given by Algorithm 2.  Then, for any $t\ge 2$,
 \begeqn\label{eq:recursive-inequ}R_{t+1}   \le    R _t - \gga_t {R_t^2 \over E_t}  +
{16\sqrt{2}\mu\gtk^2 M_t^\phi(L\gtk \wtD_t +  M_t^\phi) \gga_t\over \sqrt{t-1}}    + {4 L \gtk^2 \gga_t^2 (M_t^\phi)^2}\endeqn

\endlemm
\begin{proof}
By part (a) of Proposition \ref{prop:1}, we have
$$\begin{array}{ll} \phi((y-\wty)f_{t+1}(x,\wtx)) & \le\phi((y-\wty)f_{t}(x,\wtx)) + \langle\phi'((y-\wty)f_t(x,\wtx)) (y-\wty)K_{(x,\wtx)}, f_{t+1}-f_t\rangle_K \\
&  + { 2 L} \bigl|f_{t+1}(x,\wtx)-f_{t}(x,\wtx)\bigr|^2. \end{array}$$
Therefore, letting $\gD_t = \nabla \wtE(f_t)-{1\over t-1} \sum_{j=1}^{t-1}\int_\Z\phi'((y-y_j)f_t(x,x_j))(y-y_j)K_{(x, x_j)}d\rho(x,y),$   we know that $\EX[\wtE(f_{t+1})]$ is  bounded by
\begeqn\label{eq:recursive-2}\begin{array}{ll}  & \EX[\wtE(f_t)] - \EX \langle \nabla \wtE(f_t),  {\gga_t\over t-1} \sum_{j=1}^{t-1}\int_\Z\phi'((y-y_j)f_t(x,x_j))(y-y_j)K_{(x, x_j)}d\rho(x,y) \rangle_K\\
& + { L \gtk^4 \gga_t^2 \over (t-1)^2}\EX \bigl(  \sum_{j=1}^{t-1} |\phi'((y_t-y_j)f_t(x_t,x_j))(y_t-y_j)|\bigr)^2 \\
& \le  \EX[\wtE(f_t)] - \gga_t \EX[\|\nabla \wtE(f_t)\|_K^2]  + \gga_t \EX \langle \nabla \wtE(f_t),  \gD_t\rangle_K\\
& + { 4 L \gtk^4 \gga_t^2 \over t-1}\EX\bigl[  \sum_{j=1}^{t-1} |\phi'((y_t-y_j)f_t(x_t,x_j))|^2 \bigr]\\
& \le \EX[\wtE(f_t)] - \gga_t \EX[\|\nabla \wtE(f_t)\|_K^2]  + \gga_t \EX \bigl[ \|\nabla\wtE(f_t)\|_K\|\gD_t\|_K\bigr]\\
& + { 4 L \gtk^4 \gga_t^2 \over t-1}\EX \bigl[  \sum_{j=1}^{t-1} |\phi'((y_t-y_j)f_t(x_t,x_j))|^2\bigr]\\
& + { 4 L \gtk^4 \gga_t^2 \over t-1}\EX\bigl[  \sum_{j=1}^{t-1} |\phi'((y_t-y_j)f_t(x_t,x_j))|^2 \bigr]\\
& \le \EX[\wtE(f_t)] - \gga_t \EX[\|\nabla \wtE(f_t)\|_K^2]  + 2 \gtk \gga_t M_t^\phi \EX \bigl[\|\gD_t\|_K\bigr]\\
& + { 4 L \gtk^4 \gga_t^2 \over t-1}\EX \bigl[  \sum_{j=1}^{t-1} |\phi'((y_t-y_j)f_t(x_t,x_j))|^2\bigr]
\end{array}\endeqn
Notice
\begeqn\label{eq:est-3}\EX \bigl[  {1\over t-1}\sum_{j=1}^{t-1} |\phi'((y_t-y_j)f_t(x_t,x_j))|^2\bigr]  \le  { (M_t^\phi)^2 }.\endeqn
By the convexity of $\phi$, $\wtE(f_t) - \wtE(f_\H)\le \langle \nabla \wtE(f_t), f_t-f_\H)$ which, combined with Lemma \ref{lemm:bound-2}, implies that
$$ \EX[\|\nabla \wtE(f_t)\|_K^2]  \ge  {(\wtE(f_t) - \wtE(f_\H))^2 \over \EX[\|f_t-f_\H\|_K^2]} \ge   {R_t^2 \over E_t }.$$
Combining the above inequality, (\ref{eq:recursive-2}) and (\ref{eq:est-3}) together,  by letting
$R_t = \EX[\wtE(f_t) - \wtE(f_\H)]$, we have
$$ R_{t+1}   \le    R _t - \gga_t {R_t^2 \over E_t}  +
{16\sqrt{2}\mu\gtk^2 M_t^\phi(L\gtk \wtD_t +  M_t^\phi) \gga_t\over \sqrt{t-1}}    + {4 L \gtk^4 \gga_t^2 (M_t^\phi)^2}.$$
This completes the proof of the lemma. \end{proof}

From (3.9), in analogy to the proof  used in Theorem \ref{thm:conv-1}, one can easily see that a sufficient condition to guarantee the convergence of $ \EX[\wtE(f_t)]$ to  $\wtE(f_\H)$  can be stated as follows:
\begeqn \sum_{t=2}^\infty {\gga_t \over \ln t } = \infty,  ~~ \sum_{t=2}^\infty \Bigl[ ({M_t^\phi \wtD_t + (M_t^\phi)^2)\gga_t \over \sqrt{t-1}} + \gga_t^2 (M_t^\phi)^2\Bigr] <\infty.
\endeqn
This sufficient condition is not as neat as its counterpart to guarantee the convergence of Algorithm 1 as given by Theorem \ref{thm:conv-1}. Observe that the randomized gradient ${1\over
t-1} \sum_{j=1}^{t-1}\phi'((y_t-y_j)f_t(x_t,x_j)) (y_t-y_j)K_{(x_t,x_j)}$ in Algorithm 2 is not an unbiased estimator of the true
gradient $\iint_{\Z\times \Z}\phi'((y-y')f_t(x,x')) (y-y')K_{(x,x)}d\rho(x,y)d\rho(x',y'),$ even conditioned on $\{z_1,z_2,\ldots,z_{t-1}\}.$ This fact may partly explain why our techniques can not derive a similar sufficient condition as the one for Algorithm 1 which is stated in Theorem \ref{thm:conv-1}.

\beglemm\label{lemm:ln-inequality} For any $x, \nu, a>0$, there holds
$$ a\ln x \le \nu x +  a \ln\bigl({a \over \nu e}\bigr).$$
\endlemm
\begin{proof} The lemma directly follows from the inequality in \cite{SY}, i.e.
$e^{-\nu x} \le  \bigl({a \over \nu e }\bigr)^{a} x^{-a}.  $
\end{proof}

We are now in a position to prove Theorem \ref{thm:rate-2} by induction.

\noindent {\bf Proof of Theorem \ref{thm:rate-2}.} Denote $a_\H = \|f_\H\|_K^2 + 4 \gs_\H^2,$ and for any $\gd \in (0,\min(\gth-{1\over 2}, 1-\gth)),$ let $$\gb = \min({\gth-\gd\over 2}-{1\over 4}, 1-\gth-\gd).$$ Now let
\begeqn\label{eq:D-value} D : = \widetilde{C}_{\gth, \gd, \H}= \max\bigl\{D_1, D_2, D_3\bigr\},\endeqn
where $D_1 = {2\over c} \exp( {(1+32\gtk^4 L^2)c^2 \over 2\gth-1}) a_\H,$
$$\begin{array}{ll}D_2 & = 2^{(\gb+1)(\gth+\gb)\over \gth} \bigl({c \over a_\H}\bigr)^{\gb\over \gth}\Bigl\{2 L \gtk^2 \exp\bigl( {(1+32\gtk^4 L^2)c^2 \over 2\gth-1}\bigr) a_\H \\ &
+  \bigl( {{2 \gs_\H^2 L  \gtk^2 \exp\bigl( {(1+32 \gtk^4 L^2)c^2 \over 2\gth-1}\bigr)\over \gth} }\bigr)\bigl[ {(1+32 \gtk^4 L^2)c^2\gth \over (2\gth-1)(\gth+\gb)}+ \ln \bigl( {2^{2+\gb} c^{\gb\over \gth+\gb} \gs_\H^2 L\gtk^2   \gth^{-1} a_\H^{-{\gb\over \gth+\gb} }}\bigr)  \bigr]\\
    & 2 \gs_\H^2 L  \gtk^2 \exp\bigl( {(1+32\gtk^4 L^2)c^2 \over 2\gth-1}\bigr)  [\ln 2 +  {1\over \gth+ \gb} \ln\bigl( {c\over a_\H}\bigr) ] \Bigr\}^{\gth+\gb\over \gth}, \end{array}$$
and
$$D_3 = {2\over c}\exp\bigl( {(1+32 \gtk^4 L^2)c^2 \over 2\gth-1}\bigr) [a_\H + {2\gs_\H^2\over \gd}\ln {1\over \gd} ] + 4\gtk^2(Lc\gtk^2+8\sqrt{2}\mu) \bigl({3 L \gtk \sqrt{c}C_\phi \over \sqrt{1-\gth}} + |\phi'(0)|\bigr)^2.$$

Let $t_0 =  \bigl\lfloor 2\bigl({ c D \over 2 \exp( {(1+ 32\gtk^4 L^2)c^2 \over 2\gth-1}) a_\H }\bigr)^{1\over \gth+\gb}\bigr\rfloor. $ Since $ D \ge {2\over c} \exp( {(1+32 \gtk^4 L^2)c^2 \over 2\gth-1}) a_\H$ and $\gth+\gb\le 1$,  we have $t_0 \ge 2.$
Notice  \begeqn\label{eq:inter-inequ-0}\begin{array}{ll} & R_{t_0}   = \EX[\wtE(f_t) - \wtE(f_\H)] \le 2 L \gtk^2 \EX(\|f_{t_0} -f_\H\|_K^2) \le 2 L \gtk^2 E_{t_0}
\\ & \le 2 L \gtk^2 \exp\bigl( {(1+32 \gtk^4 L^2)c^2 \over 2\gth-1}\bigr) [a_\H +  \gs_\H^2\ln t_0 )] \\
& \le   2 L \gtk^2 \exp\bigl( {(1+32 \gtk^4 L^2)c^2 \over 2\gth-1}\bigr) a_\H +
{2 \gs_\H^2 L  \gtk^2 \exp\bigl( {(1+32 \gtk^4 L^2)c^2 \over 2\gth-1}\bigr)\over \gth} \ln D^{\gth\over \gth+\gb}  \\ &
+ 2 \gs_\H^2 L  \gtk^2 \exp\bigl( {(1+32 \gtk^4 L^2)c^2 \over 2\gth-1}\bigr)  [\ln 2 +  {1\over \gth+ \gb} \ln\bigl( {c\over a_\H}\bigr) ].
\end{array}\endeqn
Applying  Lemma \ref{lemm:ln-inequality} with $a = {2 \gs_\H^2 L  \gtk^2 \exp\bigl( {(1+32\gtk^4 L^2)c^2 \over 2\gth-1}\bigr)\over \gth}$, $\nu =2^{-1-\gb} \Bigl({2\exp\bigl( {(1+32 \gtk^4 L^2)c^2 \over 2\gth-1}\bigr)a_\H\over c }\Bigr)^{\gb\over \gth+\gb}$ and $x =D^{\gth\over \gth+\gb} $ implies that
$$\begin{array}{ll} & {2 \gs_\H^2 L  \gtk^2 \exp\bigl( {(1+32\gtk^4 L^2)c^2 \over 2\gth-1}\bigr)\over \gth} \ln D^{\gth\over \gth+\gb} \le 2^{-1-\gb} \Bigl({2\exp\bigl( {(1+32 \gtk^4 L^2)c^2 \over 2\gth-1}\bigr)a_\H\over c }\Bigr)^{\gb\over \gth+\gb}D^{\gth\over \gth+\gb} \\ &  + \bigl( {{2 \gs_\H^2 L  \gtk^2 \exp\bigl( {(1+32 \gtk^4 L^2)c^2 \over 2\gth-1}\bigr)\over \gth} }\bigr)\bigl[ {(1+32 \gtk^4 L^2)c^2\gth \over (2\gth-1)(\gth+\gb)}+ \ln \bigl( {2^{1+\gb} c^{\gb\over \gth+\gb} \gs_\H^2 L\gtk^2   \gth^{-1} a_\H^{-{\gb\over \gth+\gb} }}\bigr)  \bigr].\end{array}$$
Putting this estimation back into (\ref{eq:inter-inequ-0}), we have, for any $t\le t_0$,
\begeqn\label{eq:inter-inequ-2}\begin{array}{ll}R_t & \le  2 L \gtk^2 E_{t_0}  \le 2 L \gtk^2 \exp\bigl( {(1+32 \gtk^4 L^2)c^2 \over 2\gth-1}\bigr) a_\H \\ &
+ 2 \gs_\H^2 L  \gtk^2 \exp\bigl( {(1+32 \gtk^4 L^2)c^2 \over 2\gth-1}\bigr)  [\ln 2 +  {1\over \gth+ \gb} \ln\bigl( {c\over a_\H}\bigr) ] \\
& + 2^{-1-\gb} \Bigl({2\exp\bigl( {(1+32 \gtk^4 L^2)c^2 \over 2\gth-1}\bigr)a_\H\over c }\Bigr)^{\gb\over \gth+\gb} D^{\gth\over \gth+\gb}  \\ & + \bigl( {{2 \gs_\H^2 L  \gtk^2 \exp\bigl( {(1+32 \gtk^4 L^2)c^2 \over 2\gth-1}\bigr)\over \gth} }\bigr)\bigl[ {(1+32 \gtk^4 L^2)c^2\gth \over (2\gth-1)(\gth+\gb)}+ \ln \bigl( {2^{2+\gb} c^{\gb\over \gth+\gb} \gs_\H^2 L\gtk^2   \gth^{-1} a_\H^{-{\gb\over \gth+\gb} }}\bigr)  \bigr]\\
& \le 2^{-\gb} \Bigl({2\exp\bigl( {(1+32 \gtk^4 L^2)c^2 \over 2\gth-1}\bigr)a_\H\over c }\Bigr)^{\gb\over \gth+\gb} D^{\gth\over \gth+\gb} \le {D \over t_0^\gb} \le {D \over t^\gb},
\end{array}\endeqn
where, in the last to third inequality, we have used the fact that $D\ge D_2.$

We can now prove the theorem by induction.  Due to (\ref{eq:inter-inequ-2}),  $R_t \le {D \over t^{\gb}}$ certainly holds true for $t\le t_0.$  Now assume $R_t \le {D \over t^{\gb}}$ for some $t\ge t_0$.

To estimate $R_{t+1}$, note, by the assumption on $\phi$, that $M_t^\phi = \sup_{|t|\le 2 \gtk \wtD_t}  |\phi'(t)| \le 2 L\gtk \wtD_t + |\phi'(0)|$, and $\gga_t\le {  c \over \sqrt{t}}$ since $\gth>{1/2},$
The recursive inequality (\ref{eq:recursive-inequ}) becomes
\begeqn\label{eq:recursive-1}\begin{array}{ll}
R_{t+1}   & \le    R _t - \gga_t {R_t^2 \over E_t}  +
{32\sqrt{2}\mu\gtk^2 M_t^\phi(L\gtk \wtD_t +  M_t^\phi) \gga_t\over \sqrt{t}}    + {4 L c\gtk^4   (M_t^\phi)^2 \gga_t\over \sqrt{t}}\\
& \le R _t - \gga_t {R_t^2 \over E_t}  + { 4\gtk^2(Lc\gtk^2+8\sqrt{2}\mu)(3L\gtk \wtD_t + |\phi'(0)|)^2 \gga_t \over \sqrt{t}}
\end{array}\endeqn

Consider the function $F(x)=   x - \gga_t {x^2 \over E_t}$ which is increasing if $x\in [0, {2E_t \over \gga_t}].$
By the definition of $t_0$, it is also easy to verify, for any $t\ge t_0$, that $$ {D \over t^\gb } \le { 2 t^\gth E_t \over c} = {2 E_t \over \gga_t }.$$
Therefore, by recalling (\ref{eq:DT-gth}), i.e. $D_t \le {\sqrt{c} C_\phi \over \sqrt{1-\gth}}\, t^{1-\gth\over 2},$ we have
\begeqn\label{eq:inter-3}\begin{array}{ll}  & R_{t+1}    \le F( R _t )   +  {4\gtk^2(Lc\gtk^2+8\sqrt{2}\mu)(3 L \gtk^2 \wtD_t + |\phi'(0)|)^2 \gga_t \over \sqrt{t}}\\
&\le F(  {D \over t^{\gb}} )   +  {4\gtk^2(Lc\gtk^2+8\sqrt{2}\mu)(3 L \gtk^2 \wtD_t + |\phi'(0)|)^2 \gga_t \over \sqrt{t}}\\
& \le  {D t^{-\gb}}- \gga_t {D^2 t^{-2\gb} \over E_t}   + d_{\gth} t^{{1\over 2}-2\gth}
\end{array}\endeqn
where
 $$d_{\gth} = 4\gtk^2(Lc\gtk^2+8\sqrt{2}\mu) \bigl({3 L \gtk^2 \sqrt{c} C_\phi \over \sqrt{1-\gth}} + |\phi'(0)|\bigr)^2.$$
In addition, for any $0<\gd< \min(\gth-{1\over 2}, 1-\gth)$, applying Lemma \ref{lemm:ln-inequality} with $x = t^\gd, a=1, $ and $\nu = \gd$ implies that
$$\ln t \le t^{\gd} + {1\over \gd} \ln {1\over \gd}\le \bigl[{2\over \gd} \ln {1\over \gd}\bigr] t^\gd.$$
This yields that
$$ E_t \le \exp\bigl( {(1+32 \gtk^4 L^2)c^2 \over 2\gth-1}\bigr) [a_\H +  \gs_\H^2\ln t )] \le \exp\bigl( {(1+32 \gtk^4 L^2)c^2 \over 2\gth-1}\bigr) [a_\H + {2\gs_\H^2\over \gd}\ln {1\over \gd} ] t^\gd : = b_{\gth,\gd}\, t^\gd.$$
From the above inequality and (\ref{eq:inter-3}), and noticing $ {1\over 2} - \gth + 2\gb + \gd \le 0$, $\gth+ \gb+\gd\le 1$,  we have
\begeqn\begin{array}{ll} R_{t+1}  &  \le {D \over t^{\gb}}  \bigl[1 -  {c D    \over b_{\gth,\gd}}  t^{-\gth-\gb-\gd}   +
{d_{\gth} \over D } t^{{1\over 2}-2\gth+ \gb}\bigr] \\ &  \;= {D \over t^{\gb}}  \bigl[1 - \bigl( {c D    \over b_{\gth,\gd}}    -
{d_{\gth} \over D } t^{{1\over 2}-\gth+ 2\gb+ \gd} \bigr) t^{-\gth-\gb-\gd}\bigr]\\
& \;\le {D \over t^{\gb}}  \bigl[1 - \bigl( {c D    \over b_{\gth,\gd}}    -
{d_{\gth} \over D }  \bigr) t^{-\gth-\gb-\gd}\bigr] \\
& \le {D \over t^{\gb}}  \bigl[1 -   t^{-\gth-\gb-\gd}\bigr] \le {D \over t^{\gb}}  \bigl[1 -   t^{-1}\bigr]\\ & \le  {D \over t^{\gb}}  \bigl[1 -   {(t+1)}^{-1}\bigr] \le {D \over (t+1)^{\gb}}, \end{array}\endeqn
where the last to fourth inequality used the fact that ${c D    \over b_{\gth,\gd}}    -
{d_{\gth} \over D } \ge 1 $ since $D \ge D_3 = {2b_{\gth,\gd}\over c} + d_\gth \ge  {1\over 2}\bigl({b_{\gth,\gd}\over c } + \sqrt{ {b_{\gth,\gd}^2 \over c^2} + {4 b_{\gth,\gd} d_{\gth}\over c }  }\bigr).$  This completes the proof  of the theorem.  \hfill $\Box$

We turn our attention to the proof of Theorem \ref{thm:rate-3}.

\noindent {\bf Proof of Theorem \ref{thm:rate-3}.}  For any $\gd \in (0,\min({\gth\over 4}, 1-\gth)),$ and  let $$\gb = \min({\gth\over 4}-{\gd\over 2}, 1-\gth-\gd).$$ Let $D_1, D_2$ and $t_0$ be the same as those introduced in the proof for Theorem \ref{thm:rate-2}.  Choose
$ D : = \bar{C}_{\gth, \gd, \H}= \max\bigl\{D_1, D_2, \wtD_3\bigr\},$
where
$$\wtD_3   = {2\over c}\exp\bigl( {(1+2 \gtk^4 L^2)c^2 \over 2\gth-1}\bigr) [a_\H + {2\gs_\H^2\over \gd}\ln {1\over \gd} ] + {4\gtk^2 B (8\sqrt{2}\mu L\gtk {\sqrt{c} C_\phi\over \sqrt{1-\gth}} +  (8\sqrt{2}\mu+Lc\gtk^2)B)}.
$$
Since $|\phi'(s)|\le B $ for any $s\in \R$, $M_t^\phi \le B$ holds true uniformly.    Hence, for any $t \le t_0 =  \bigl\lfloor 2\bigl({ c D \over 2 \exp( {(1+ 32\gtk^4 L^2)c^2 \over 2\gth-1}) a_\H }\bigr)^{1\over \gth+\gb}\bigr\rfloor,$ there holds
$R_t \le {D\over t_0^\gb}\le {D \over t^\gb}.$ Assume that, for some $t\ge t_0$, $R_t \le  {D \over t^\gb}.$ We will prove that $R_{t+1} \le {D\over (t+1)^\gb}$ by induction. To this end,     observing that   $M_t^\phi \le B$ holds true uniformly, we know from  the recursive inequality (\ref{eq:recursive-1}) that
$$\begin{array}{ll} R_{t+1}   & \le    R _t - \gga_t {R_t^2 \over E_t}  +
{32\sqrt{2}\mu\gtk^2 M_t^\phi(L\gtk \wtD_t +  M_t^\phi) \gga_t\over \sqrt{t}}    + {4 L c\gtk^4   (M_t^\phi)^2 \gga_t\over \sqrt{t}}\\
& \le  R _t - \gga_t {R_t^2 \over E_t}  + {4\gtk^2 B \bigl[8\sqrt{2}\mu L \gtk \wtD_t + (8\sqrt{2}\mu+Lc\gtk^2)B\big] \gga_t\over \sqrt{t}}
\end{array}$$
Recalling (\ref{eq:DT-gth}) again, i.e. $\wtD_t \le {\sqrt{c} C_\phi \over \sqrt{1-\gth}}\, t^{1-\gth\over 2},$ we have
\begeqn\label{eq:inter-4}\begin{array}{ll}  & R_{t+1}    \le F( R _t )   +  { 4\gtk^2 B \bigl[8\sqrt{2}\mu L \gtk \wtD_t + (8\sqrt{2}\mu+Lc\gtk^2)B\big] \gga_t \over \sqrt{t}}\\
& \le  {D t^{-\gb}}- \gga_t {D^2 t^{-2\gb} \over E_t}   + \wtd_{\gth} t^{-{3\gth\over 2}}
\end{array}\endeqn
where
 $$\wtd_{\gth} =  {4\gtk^2 B (8\sqrt{2}\mu L\gtk {\sqrt{c} C_\phi\over \sqrt{1-\gth}} +  (8\sqrt{2}\mu+Lc\gtk^2)B)}.$$
In analogy to the argument in the proof of Theorem \ref{thm:rate-2}, from the above inequality and (\ref{eq:inter-4}), and noticing $ - {\gth\over 2} + 2\gb + \gd \le 0$, $\gth+ \gb+\gd\le 1$,  we have
\begeqn\begin{array}{ll} R_{t+1}  &  \le {D \over t^{\gb}}  \bigl[1 -  {c D    \over b_{\gth,\gd}}  t^{-\gth-\gb-\gd}   +
{\wtd_{\gth} \over D } t^{-{3\gth\over 2}+ \gb}\bigr] \\ &  \;= {D \over t^{\gb}}  \bigl[1 - \bigl( {c D    \over b_{\gth,\gd}}    -
{\wtd_{\gth} \over D } t^{-{\gth\over 2}+ 2\gb+ \gd} \bigr) t^{-\gth-\gb-\gd}\bigr]\\
& \;\le {D \over t^{\gb}}  \bigl[1 - \bigl( {c D    \over b_{\gth,\gd}}    -
{\wtd_{\gth} \over D }  \bigr) t^{-\gth-\gb-\gd}\bigr] \\
& \le {D \over t^{\gb}}  \bigl[1 -   t^{-\gth-\gb-\gd}\bigr] \le {D \over t^{\gb}}  \bigl[1 -   t^{-1}\bigr]\\ & \le  {D \over t^{\gb}}  \bigl[1 -   {(t+1)}^{-1}\bigr] \le {D \over (t+1)^{\gb}}, \end{array}\endeqn
where the last to fourth inequality used the fact, by the fact that $D\ge \wtD_3 = {2b_{\gth,\gd}\over c} + \wtd_{\gth},$ which means that ${c D    \over b_{\gth,\gd}}    -
{\wtd_{\gth} \over D } \ge 1 .$  This completes the proof of the theorem.  \hfill $\Box$

\section{Conclusion}
In this paper, we considered the unregularized online learning algorithms in the RKHSs for both classification and pairwise learning problems associated with general loss functions. We derived sufficient conditions on the step sizes to guarantee their convergence, and established explicit convergence rates with polynomially decaying step sizes. This is in contrast to most of studies which are mainly focused on regularized online learning \cite{SY,TY,YeZ,YZ}. Our novel results are obtained by using  tools from convex analysis, refined properties of Rademacher averages and an smart  induction approach.   Below, we discuss some directions for future work.

Firstly, the rates for Algorithm 1 and Algorithm 2 are suboptimal. For instance, in the special case of the least-square loss, it was proved in \cite{YP} that Algorithm 1 can achieve $\O(T^{-{1\over 2}}\ln T)$ if $f_\rho\in \H_G$. However, by Theorem \ref{thm:rate-1}, the rate is only of $\O(T^{-{1\over 3}}).$  It remains an open and challenging question on how to improve the rates for unregularized online learning algorithms with general loss functions.  Secondly, our main theorems assume that $g_\H = \arg\inf_{g\in \H_G}\E(g)$ and $f_\H =\arg\inf_{f\in \H_K}\wtE(f)$ exist.  However, we know from \cite{YP,YZ2} that   this assumption can be removed for the least-square loss. It is a clearly important future work to discuss whether this assumption will also be removed for   general loss functions.  Thirdly, the techniques in this paper rely some smoothness assumptions on the loss function, and hence can not handle the popular hinge loss. It remains an open question to us how to establish the convergence of unrgularized online learning algorithms associated with the hinge loss.
Lastly, our results are established in the form of expectation.  It would be interesting to prove the almost surely convergence of the last iterate of Algorithms 1 and 2.

\section*{Acknowledgements}

We would like to thank the referees for their invaluable comments and suggestions.  We are also grateful to Dr. Yunwen Lei for pointing out a bug in the proof of Lemma 5 in an early version of the paper and providing Lemma 6 to us. The work by D. X. Zhou described in this paper is supported by a grant from the Research Grants Council of Hong Kong [Project No.
CityU 105011].

\end{document}